\documentclass[11pt,twoside]{article}
\usepackage{fullpage}
\usepackage{epsf}
\usepackage{fancyhdr}
\usepackage{graphics}
\usepackage{graphicx}
\usepackage{psfrag}
\usepackage[T1]{fontenc}
\usepackage{color}
\usepackage{amsthm}
\usepackage{amsfonts}
\usepackage{amsmath}
\usepackage{amssymb}
\usepackage{mathrsfs}
\usepackage{bm}
\usepackage{accents}
\usepackage{listings}
\usepackage{caption}
\usepackage{subcaption}
\usepackage[linesnumbered, ruled, vlined]{algorithm2e}
\usepackage[titletoc,toc,title]{appendix}
\usepackage{enumerate}
\usepackage{verbatim} 
\usepackage{csquotes} 
\usepackage{upquote} 
\usepackage[bottom]{footmisc} 
\usepackage{thmtools}
\usepackage{thm-restate}
\usepackage[dvipsnames]{xcolor}
\usepackage[english]{babel} 
\usepackage[shortlabels]{enumitem}
\usepackage{mathtools}
\usepackage[colorlinks,linkcolor = RawSienna, urlcolor  = RedViolet, citecolor = RoyalBlue, anchorcolor = ForestGreen,bookmarks=False]{hyperref}
\usepackage{cleveref}

\newlength{\widebarargwidth}
\newlength{\widebarargheight}
\newlength{\widebarargdepth}

\makeatletter
\long\def\@makecaption#1#2{
        \vskip 0.8ex
        \setbox\@tempboxa\hbox{\small {\bf #1:} #2}
        \parindent 1.5em  
        \dimen0=\hsize
        \advance\dimen0 by -3em
        \ifdim \wd\@tempboxa >\dimen0
                \hbox to \hsize{
                        \parindent 0em
                        \hfil 
                        \parbox{\dimen0}{\def\baselinestretch{0.96}\small
                                {\bf #1.} #2
                                } 
                        \hfil}
        \else \hbox to \hsize{\hfil \box\@tempboxa \hfil}
        \fi
        }
\makeatother

\renewcommand{\baselinestretch}{1.04} 
\date{}
\usepackage[style = alphabetic,citestyle=alphabetic,maxbibnames=99,backend=biber,sorting=nyt,natbib=true,backref=true]{biblatex}
\DefineBibliographyStrings{english}{%
  backrefpage = {Cited on page},
  backrefpages = {Cited on pages},
}
\newcommand{\BlackBox}{\rule{1.5ex}{1.5ex}}  

\renewenvironment{proof}{\par\noindent{\bf Proof\ }}{\hfill\BlackBox\\[2mm]}
\newtheorem{theorem}{Theorem}

\newtheorem{lemma}[theorem]{Lemma}

\newtheorem{definition}[theorem]{Definition}

\newcommand\numberthis{\addtocounter{equation}{1}\tag{\theequation}}

\newcommand*{\eqdef}{:=}

\newcommand{\ceil}[1]{{\left\lceil #1 \right\rceil}}

\newcommand{\cK}{{\cal K}}

\newcommand{\cN}{{\cal N}}

\newcommand{\R}{\mathbb{R}}
\renewcommand{\S}{\mathbb{S}}
\newcommand{\N}{\mathbb{N}}

\newcommand{\E}{\mathbb{E}}

\renewcommand{\Pr}{\mathbb{P}}
\newcommand{\lv}{\lVert}
\newcommand{\rv}{\rVert}

\newcommand{\tI}{{\widetilde{I}}}

\newcommand{\tx}{\tilde{x}}
\newcommand{\tmu}{\tilde{\mu}}
\newcommand{\tW}{\widetilde{W}}
\newcommand{\tV}{\widetilde{V}}
\newcommand{\tv}{\tilde{v}}
\newcommand{\tL}{\widetilde{L}}
\newcommand{\tsigma}{\widetilde{\sigma}}
\newcommand{\talpha}{\tilde{\alpha}}
\renewcommand{\epsilon}{\varepsilon}
\renewcommand{\ln}{\log}

\DeclareMathSizes{9}{8}{7}{5}
\title{\textbf{When Does Gradient Descent 
       with Logistic Loss \\
       Find Interpolating Two-Layer Networks?}}
\author{Niladri S. Chatterji \\ University of California, Berkeley \\ \texttt{chatterji@berkeley.edu} \and Philip M. Long \\ Google \\ \texttt{plong@google.com}  \and Peter L. Bartlett \\ University of California, Berkeley \& Google \\ \texttt{peter@berkeley.edu}}
\date{\today}
\begin{document}
\maketitle
\begin{abstract}
We study the training of finite-width two-layer smoothed ReLU networks for binary classification using the logistic loss. We show that gradient descent drives the training loss to zero if the initial loss is small enough. When the data satisfies certain cluster and separation conditions and the network is wide enough, we show that one step of gradient descent reduces the loss sufficiently that the first result applies. 
\end{abstract}

\section{Introduction}

The success of deep learning 
has led to a lot of recent interest in understanding the properties
of ``interpolating'' neural network models, that achieve 
(near-)zero training loss \citep{zhang2016understanding,belkin2019reconciling}. 
One aspect of understanding these models is to theoretically characterize how first-order gradient methods (with appropriate random initialization) seem to reliably find interpolating solutions to non-convex optimization problems.  

In this paper, we show that, under two sets of conditions, 
training fixed-width two-layer networks with
gradient descent drives the logistic loss to zero.
The networks have
smooth ``Huberized''
ReLUs \citep[][see equation~\eqref{e:helu} and Figure~\ref{f:hrelu}]{tatro2020optimizing}
and 
the output weights are not trained.

The first result only requires the assumption that the initial loss is small, but does not require any assumption about either the width of the network or the number of samples. It guarantees that if the initial loss is small then gradient descent drives the logistic loss to zero.

For our second result we 
assume that the 
inputs come
from four clusters, 
two per class,
and that the clusters corresponding to the opposite labels are appropriately separated. Under these assumptions, we show that random Gaussian initialization along with a single step of gradient descent is enough to guarantee that the loss reduces sufficiently that the first result applies.

A few proof ideas that facilitate our results are as follows: under our first set of assumptions, when the loss is small, we show that the negative gradient
aligns well with the parameter vector. This yields a lower bound on
the norm of the gradient in terms of the loss and the norm
of the current weights.  This implies that, if the
weights are not too large, the loss is reduced rapidly
at the beginning of
the gradient descent step.  
Exploiting the Huberization of the ReLUs, we also show
that the loss is a smooth function of the weights, so
that the loss continues to decrease rapidly throughout
the step, as long as the step-size is not too big.
Crucially, we show that the loss is decreased significantly
compared with the size of the change to the weights.
This implies, in particular, that the norm of the
weights does not increase by too much, so
that progress can continue.

The preceding analysis requires a small loss
to ``get going''.  Our second result provides one example when this provably happens.  A two-layer network may be viewed as
a weighted vote over predictions made by the hidden
units.  Units only vote on examples that fall in
halfspaces where their activation functions are non-zero.
When the network
is randomly initialized, we can think of each hidden unit
as ``capturing'' roughly half of the examples---each
example is turn captured by roughly half of the hidden
units.  Some capturing events are helpful, and some
are harmful.  At initialization, these are roughly
equal.  
Using
the properties of the Gaussian 
initialization (including concentration and
anti-concentration) we show that each example $(x_s,y_s)$ is
captured by many nodes whose
first updates contribute to improving its
loss.
For this to happen, the updates for this
example must not be offset by updates for other
examples.  This happens with sufficient probability at
each individual node
that
the cumulative effect of these
``good'' nodes overwhelms the effects of
potentially confounding nodes, which tend to
cancel one another.  Consequently, 
with $2p$ hidden nodes,
the
loss after one iteration is at most 
$\exp(-\Omega(p^{1/2 - \beta}))$ for
$\beta > 0$.  
By comparison, under similar, but weaker,
clustering assumptions,  \citet{li2018learning} used a neural tangent kernel (NTK) analysis
to show that the loss is $1/\textsf{poly}(p)$
after $\textsf{poly}(p)$ steps.
Our proof uses more structure of the problem than the NTK proof, for
example, that (loosely speaking) the reduction in
the loss is exponential in the number of
hidden units improved.

We work with smooth Huberized ReLUs to facilitate theoretical analysis.
We analyze networks with Huberized ReLUs instead of the
increasingly popular Swish \citep{ramachandran2017searching}, which
is also a smooth approximation to the ReLU, to
facilitate a simple analysis.
We describe some preliminary experiments with
artificial data supporting our theoretical analysis,
and suggesting that networks with Huberized ReLUs
behave similarly to networks with standard ReLUs.

Related results, under weaker assumptions,
have been obtained for the quadratic loss 
\citep{du2018gradient,du2019gradient,allen2019convergence,oymak2020towards}, using the NTK~\citep{jacot2018neural,chizat2019lazy}.
The logistic loss is qualitatively different;  among other things,
driving the logistic loss to zero requires the weights to 
go to infinity, far from their initial values, so that a
Taylor approximation around the initial values cannot be applied.
The NTK framework 
has also been applied to analyze training with the
logistic loss.
A typical result \citep{li2018learning,allen2019convergence,zou2020gradient}
is that
after $\textsf{poly}(1/\epsilon)$ 
updates, a network of size/width $\textsf{poly}(1/\epsilon)$ 
achieves $\epsilon$ loss. Thus to guarantee loss very close to zero, these
analyses require larger and larger networks. The reason for this
appears to be that a key part of these analyses
is to show that a wider network can achieve a certain fixed loss by traveling
a shorter distance in parameter space.  Since it seems that, to drive the
logistic loss to zero with a fixed-width network, the parameters must travel 
an unbounded distance, the NTK approach cannot be applied to obtain the
results of this paper.

In a recent paper, \citet{lyu2019gradient} studied the margin maximization of ReLU networks 
for 
the logistic loss.
\citet{lyu2019gradient} also proved the
convergence of gradient descent to zero, but that result requires positive homogeneity and smoothness, which rules out the ReLU and similar nonlinearities like the Huberized ReLU studied here.
Their results do apply in the case that the ReLU is raised
to a power strictly greater than two.  
Lyu and Li used both assumptions of
positive homogeneity and smoothness to prove the results
in their paper that are most closely related to this paper, so that a substantially different analysis
was needed here.  (See, for example, the proof of Lemma E.7 of their
paper.)  As far as we know, the analysis of the alignment
between the negative gradient and the weights
originated in their paper: in this paper, we establish such
alignment under weaker conditions.

Building on this work by \citet{lyu2019gradient}, \citet{ji2020directional} studied finite-width deep ReLU neural networks and showed that starting from a small loss, gradient flow coupled with logistic loss leads to convergence of the directions of the parameter vectors. They also demonstrate alignment between the parameter vector directions and the negative gradient. However, they 
do not prove that the training loss converges to zero.



The remainder of the paper is organized as follows. In Section~\ref{s:defs} we introduce notation,
definitions, assumptions, and present both of our main results. 
We provide a proof of Theorem~\ref{t:main.general} in Section~\ref{s:proofdetails_main_general} and we prove Theorem~\ref{t:main.clusters} in Section~\ref{s:proofdetails_theorem_cluster}. Section~\ref{s:simulations} is devoted to some numerical simulations.
Section~\ref{s:furtherrelatedwork} points to other related work and we conclude with a discussion in Section~\ref{s:discussion}. 

\section{Preliminaries and Main Results}
\label{s:defs}

This section includes notational conventions, a description of
the setting, and the statements of the main results.

\subsection{Notation}
Given a vector $v$, let $\lv v \rv$ denote its Euclidean norm. Given a matrix $M$, let $\lv M \rv$ denote its Frobenius norm and $\lv M \rv_{op}$ denote its operator norm.  For any $k \in \N$, we denote the set $\{ 1,\ldots,k \}$ by $[k]$. For a number $d$ of inputs, we denote
the set of unit-length vectors
in $\R^d$ by $\S^{d-1}$. Given an event $A$, we let $\mathbf{1}_A$ denote the indicator of this event. 
The symbol $\wedge$ is used to denote the logical ``AND" operation. At multiple points in the proof we will use the standard ``big Oh notation'' \citep[see, e.g.,][]{cormen2009introduction} to denote how certain quantities scale with the number of hidden units ($2p$), while viewing all other problem parameters that are not specifically
set as a function of $p$ as constants.
We will use $C_1, C_2, \ldots$ to denote absolute
constants whose values are fixed throughout the
paper, and $c', c_1, \ldots$ to denote ``local''
constants, which may take different values in
different contexts.

\subsection{The Setting}

We will analyze gradient descent applied
to minimize the training loss of a
two-layer network.  

Let $d$ be the number of inputs, and
$2p$ be the number of hidden nodes. We consider the
case that the weights connected to the
output nodes are fixed: $p$ of them
take the value $1$, and the other $p$
take the value $-1$.  

We work with Huberized ReLUs that are defined as follows:
\begin{equation}
\label{e:helu}
\phi(z) := \left\{ \begin{array}{ll}
            0 & \mbox{ if $z < 0$,} \\
            \frac{z^2}{2 h} & \mbox{ if $z \in [0,h]$,} \\
            z - h/2 & \mbox{otherwise.}
         \end{array}
           \right.
\end{equation}
See Figure~\ref{f:hrelu}. 
\begin{figure}
  \centering
  \includegraphics[width=4in]{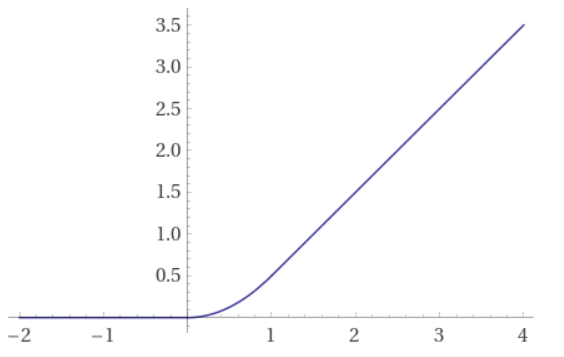}
  \caption{\label{f:hrelu} 
  A plot of the Huberized ReLU in the case $h=1.$
   }
  \end{figure}
We set the value of the bandwidth parameter $h=1/p$ throughout the paper.

For $i \in \{1,\ldots,2p\}$, let
$v_{i} \in \R^d$ be vector of weights from
the inputs to the $i$th hidden node,
and let $u_1,\ldots,u_{2p} \in \R$ be the
weights connecting the hidden nodes
to the output node. Set $u_1 = \ldots = u_p = 1$ 
and $u_{p+1} = \ldots = u_{2p} = -1$. Let $b_i$ be the bias for
the $i$th hidden node.  Let $\theta = (v_1,b_1,\ldots,v_{2p},b_{2p})$
consist of all of the trainable parameters in the network.
Let $f_{\theta}$ denote the function computed by the 
network, which maps $x$ to
\[
f_{\theta}(x) = \sum_{i=1}^{2p} u_i \phi(v_{i} \cdot x + b_i).
\]

Consider a training set $(x_1,y_1),\ldots,(x_n,y_n) \in \S^{d-1} \times \{-1, 1 \}$.  Define the training loss $L$ by
\[
L(\theta) := \frac{1}{n}\sum_{s=1}^n \ln(1+\exp\left(-y_s f_{\theta}(x_s)\right)),
\]
and refer to loss on example $s$ by
\[
L(\theta; x_s,y_s) := \ln(1+\exp\left(-y_s f_{\theta}(x_s)\right)).
\]
The gradient of the loss evaluated at $\theta$ is
\begin{align*}
    \nabla_{\theta} L(\theta) = \frac{1}{n} \sum_{s=1}^n \frac{-y_s \nabla_{\theta} f_{\theta}(x_s)}{1+\exp\left(y_s f_{\theta}(x_s)\right)}.
\end{align*}
Note that, since
$\phi''$ is not defined
at $0$ and $h$, the Hessian
of $L$ is not defined
everywhere.  We use
the following weak derivative
of $\phi'$ to define a weak
Hessian ($\nabla^2_{\theta} L(\theta)$):
\begin{align}
\label{e:hessian_helu}
\gamma(z) := \left\{ \begin{array}{ll}
            \frac{1}{ h} & \mbox{ if $z \in [0,h]$,} \\
            0 & \mbox{otherwise.}
         \end{array}
           \right.
\end{align}
\subsection{A General Bound}
\label{s:main.general}
We first analyze the iterates $\theta^{(1)},\theta^{(2)},\ldots$ defined by
\[
\theta^{(t+1)} := \theta^{(t)} - \alpha_t \nabla_{\theta} L\lvert_{\theta = \theta^{(t)}},
\]
in terms of properties of $\theta^{(1)}$.
\begin{theorem}
\label{t:main.general} 
There is an absolute constant $C_1 > 0$ such that, 
for all $n \geq 2$, 
for all 
$p \geq 1$,
for any initial parameters $\theta^{(1)}$
and
dataset $(x_1,y_1),\ldots,(x_n,y_n) \in \S^{d-1} \times \{-1,1\}$, for any positive $Q_1 \le \widetilde{Q}_1$ and positive $Q_2 \le \widetilde{Q}_2(Q_1)$ (where $\widetilde{Q}_1$ and $\widetilde{Q}_2$ are defined in eqs.~\eqref{eq:Q_1_def}-\eqref{eq:Q_2_def}) the
following holds for all
$t \geq 1$.
If $h = 1/p$ and 
each step-size
$\alpha_t = Q_1 \log^2(1/L(\theta^{(t)}))$,
and if $L(\theta^{(1)})\leq \frac{1}{n^{1 + C_1}}$
then, for all $t \geq 1$,
        \[
L(\theta^{(t)}) 
 \le 
    \frac{L(\theta^{(1)})}{ Q_2 \cdot (t-1)+1}.
\]
\end{theorem}
The proof of this theorem is presented in Section~\ref{s:proofdetails_main_general} below. 

We reiterate that this theorem makes no assumption about the number of hidden nodes ($p$) and makes a mild assumption on the number of samples required ($n\ge 2$). The only other constraint is that the initial loss needs to be smaller than $1/n^{1+C_1}$, for some universal constant $C_1>0$. Our choice of a time-varying step-size $\alpha_t$ that grows with $\log^2(1/L(\theta^{(t)}))$ leads to an upper bound on the loss that scales with $1/t$.

\subsection{Clustered Data with Random Initialization}
\label{s:main.clusters}
We next consider gradient descent after random
initialization by showing that, after one iteration,
$\theta^{(1)}$ has the favorable properties needed to apply
Theorem~\ref{t:main.general}.
We assume that all trainable parameters are initialized by
being chosen independently at random from
$\cN(0,\sigma^2)$.  
Let $\theta^{(0)}$ be the initial value of
the parameters and $\alpha_0$ be the original
step-size (which will be chosen as a function of $p$).

This analysis depends on cluster and separation conditions. We shall use $k$ and $\ell$ to index over the clusters (ranging from $1$ to $4$), and $s$ and $q$ to index over individual samples (ranging from $1$ to $n$).
We assume that the training data can be divided into four clusters $\cK_1,\ldots,\cK_4$. All examples $(x_s,y_s)$ in
clusters $\cK_1$ and $\cK_2$ have $y_s = 1$ and 
all examples $(x_s,y_s)$ in
clusters $\cK_3$ and $\cK_4$ have $y_s = -1$. For some cluster index $k$, let $y_k$ be the label shared by
all examples in cluster $\cK_k$. With some abuse of notation we will often use $s\in \cK_k$ to denote that the example $(x_s,y_s)$ belongs to the cluster $\cK_k$. 

We make the following assumptions about the clustered training data.  
\begin{itemize}
\item For $\epsilon > 0$, for
each cluster $\cK_k$, we assume $n_k \eqdef |\cK_k|$ satisfies 
$\left(1/4-\epsilon\right)n \leq n_k \leq \left(1/4+\epsilon\right)n.$%
\item Assume that
$\lv x_s \rv=1$ for all $s \in [n]$.
\item For a radius $r > 0$,
each cluster $\cK_k$ has a center $\mu_k$ with $\lv \mu_k\rv = 1$, such that
for all $s \in \cK_k$, $\lv x_s - \mu_k \rv \leq r$.
\item For a separation parameter $\Delta > 0$, we assume that
for all $k,\ell$ with $y_k \neq y_\ell$, $\mu_k \cdot \mu_{\ell} \leq \Delta$.
\end{itemize}

Under these assumptions we demonstrate that with high probability random initialization followed by one step of gradient descent leads to a network whose training loss is at most $\exp\left(-\Omega(p^{1/2-\beta})\right)$ for $\beta >0$. 
Theorem~\ref{t:main.general}
then implies that gradient descent in the subsequent steps leads to a solution with
training loss approaching zero.

\begin{theorem}
\label{t:main.clusters}
For any $\beta \in (0,1/2)$,
there are absolute constants $C_2,\ldots,C_6 > 0$ such that, under the assumptions described above, the
following holds for all $\delta < 1/2$ and $n \geq 4$.
If 
$\Delta < C_2$, 
$r < C_2$, $\epsilon < C_2$,
$\alpha_0 = \frac{1}{p^{1/2 + \beta}}$, $\sigma = \frac{1}{p^{1/2 + \beta/2}}$, $h = \frac{1}{p}$, and
$p \geq \ln^{C_3}(n d/\delta)$ both of the following hold:
\begin{enumerate}[(a)]
   \label{i:first_step}
    \item With probability $1-\delta$, $$L(\theta^{(1)}) 
     \le \exp(-C_4 p^{(1/2 - \beta)});$$
    \item 
\label{i:fixed_alpha} 
    if, for all $t \geq 1$, $\alpha_t = \frac{C_5\log^2(1/L(\theta^{(t)}))}{p}$, then
      with probability $1 - \delta$, 
       for all $t > 1$,
        \[
L(\theta^{(t)}) 
 \le 
    \frac{C_6 d}{ p^{1 - 3 \beta} t}.
\]
\end{enumerate}
\end{theorem}
This theorem is proved in Section~\ref{s:proofdetails_theorem_cluster}.
It 
shows that if the data satisfies the cluster and separation conditions then the loss after a single step of gradient descent decreases by an amount that is exponential in $p^{1/2-\beta}$ with high probability. This result only 
 requires the width $p$ to be poly-logarithmic in the number of samples, input dimension and $1/\delta$.


\section{Proof of Theorem~\protect\ref{t:main.general}}
\label{s:proofdetails_main_general}

In this section, we prove
Theorem~\ref{t:main.general}.
Our proof is by induction.  As mentioned in
the introduction, the key lemma is a lower bound on the norm of the gradient.
Our lower bound (Lemma~\ref{l:lower.bound.gradient}) is in terms of the
loss, and also the norm of
the weights.  Roughly speaking, for it to provide leverage, we need that the loss
is small relative to the size of the weights, or, in other words, that the
model doesn't excessively ``waste weight''.  The
bound of Theorem~\protect\ref{t:main.general} accounts for the amount of such
wasted weight at initialization, so we do not need a wasted-weight assumption.  On the other hand, we need a condition on the wasted weight
in our inductive hypothesis---we need to prove that training does not increase
the amount of wasted weight too much.

Lemma~\ref{l:lower.bound.gradient} also requires an upper bound on the step
size -- another part of the inductive hypothesis ensures that this requirement
is met throughout training.

Before the proof, we lay some
groundwork.  First, to simplify expressions, we reduce to
the case that the biases are zero.  Then we establish some lemmas that
will be used in the inductive step, about the progress in an iteration,
smoothness, etc.  Finally, we applied these tools in the inductive proof.

\subsection{Reduction to the Zero-Bias Case}
We first note that, applying a standard reduction,
without loss of generality, we may assume
\begin{itemize}
    \item $b_1,\ldots,b_{2p}$ are fixed to $0$, and not trained,
  and
    \item for all $s, q$, $x_s \cdot x_{q} \geq 0$.
\end{itemize}
The idea is to adopt the view that the inputs have an
additional component that acts as a placeholder for the bias
term, which allows us to view the bias term as another component
of $v_i$.   The details are in Appendix~\ref{a:reduction.general}.  
We will make the above assumptions from now on.  Since the bias
terms are fixed at zero, for a matrix $V$ whose rows are the weights of the hidden units,
we will refer to the resulting loss as $L(V)$, 
$f_{\theta}$ as $f_V$, and so on.
Let $V^{(t)}$ be the $t$th iterate.

\subsection{Additional Definitions}

\begin{definition}
\label{d:L_tk}
For all 
iterates $t$,
define
$L_{ts} := L(V^{(t)}; x_s,y_s)$ and let $L_t := \frac{1}{n} \sum_{s=1}^n L_{ts}$. 
Additionally for all 
$t$,
define $\nabla L_t : = \nabla_{V} L|_{V = V^{(t)}}$. We will also use $\nabla_{V^{(t)}}L$ to refer to the gradient $\nabla_{V} L|_{V = V^{(t)}}$.
\end{definition}
\begin{definition}
\label{d:g_tk} For any weight matrix $V$, define
\begin{align*}
    g_s(V):= \frac{1}{1+\exp\left(y_s f_{V}(x_s)\right)}.
\end{align*}
We often 
will use $g_s$ as shorthand for $g_s(V)$ when
$V$ can be determined from context.
Further, for all $t \in \{0,1,\ldots\}$, define
$g_{ts} := g_s(V^{(t)}).$
\end{definition}
Informally, $g_s(V)$ is the size of the contribution
of example $s$ to the gradient.

\subsection{Technical Tools}\label{ss:technicaltool}
In this subsection we assemble several technical tools required to prove Theorem~\ref{t:main.general}. The proofs that are omitted in this subsection are presented in Appendix~\ref{a:omiittedproof_main}.

We start with the following lemma, which is a slight variant of a standard inequality, and provides a bound on the loss after a step of gradient descent when the loss function is locally smooth.
It is proved in Appendix~\ref{a:pl}.
\begin{restatable}{lem}{pllemma}
\label{l:pl}
For $\alpha_t > 0$, let $V^{(t+1)} = V^{(t)} - \alpha_t \nabla L_t$.
If, for all convex combinations $W$ of $V^{(t)}$
and $V^{(t+1)}$, we have $\lv \nabla^2_{W} L \rv_{op} \leq M$, then
if $\alpha_t \leq \frac{1}{3M}$, we have
\[
L_{t+1} \leq L_t - \frac{5\alpha_t \lv \nabla L_t \rv^2}{6}.
\]
\end{restatable}

To apply Lemma~\ref{l:pl} we need to show that the loss function $L$ is smooth
near $L_t$; the following lemma is a start.
It is proved in Appendix~\ref{a:L_smooth}.
\begin{restatable}{lem}{hessianlemma}
\label{l:L_smooth} 
If $h=1/p$, for any weight matrix $V \in \R^{2p \times (d+1)}$, 
$
\lv \nabla_{V}^2 L \rv_{op} \le 5  p L(V).
$
\end{restatable}
Next, we show that $L$ 
changes slowly in general, and especially slowly when it is
small.  The proof is in Appendix~\ref{a:gradient_norm.upper}.

\begin{restatable}{lem}{gradientupper}
\label{l:gradient_norm.upper} For any weight matrix $V \in \R^{2p \times (d+1)}$, 
$\lv \nabla_{V} L \rv \le \sqrt{2p} \min\{ L(V), 1\}.$
\end{restatable}
The following lemma applies Lemma~\ref{l:pl} (along with Lemma~\ref{l:L_smooth}) to show that if the step-size at step $t$ is small enough then the loss decreases by an amount that is proportional to the squared norm of the gradient. 
Its proof is in Appendix~\ref{a:L.one_step_improvement}.
\begin{restatable}{lem}{onesteplemma}
\label{l:L.one_step_improvement}
If $\alpha_t L_t \leq \frac{1}{30  p }$, then
$
L_{t+1} \leq L_t- \frac{5\alpha_t  \lv \nabla L_t \rv^2 }{6} .
$
\end{restatable}
We need the following technical lemma which is proved in Appendix~\ref{a:concave.min}.
\begin{restatable}{lem}{concaveminlemma}
\label{l:concave.min}
If $\psi:(0,M] \to \mathbb{R}$ is a continuous, concave function such that $\lim_{\eta \to 0^{+}} \psi(\eta)$ exists. Then the infimum of
$\sum_{i=1}^n \psi(z_i)$ subject to
$z_1,\ldots,z_n > 0$ and
$\sum_{i=1}^n z_i = M$ is $\psi(M)+(n-1)\lim_{\eta \to 0^{+}} \psi(\eta)$.
\end{restatable}
The next lemma establishes a lower bound on the norm of the gradient of the loss in the later iterations. 
\begin{lemma}
\label{l:lower.bound.gradient} 
For all large enough $C_1$, for any $t \ge 1$,
if $L_t \le 1/n^{1+C_1}$,
then
\begin{align} \label{e:lowerboundonthegradient_used}
    \lv \nabla L_t \rv \ge \frac{5L_t \log(1/L_t)}{6 \lv V^{(t)} \rv}.
\end{align}
\end{lemma}
\begin{proof}
Since
\begin{align*}
\lv \nabla L_t \rv =\sup_{a: \lv a \rv =1} \left(\nabla L_t \cdot a\right)\geq (\nabla L_t) \cdot\left(\frac{-V^{(t)}}{\lv V^{(t)} \rv}\right),
\end{align*}
we seek a lower bound on $-\nabla L_t \cdot \frac{V^{(t)}}{\lv V^{(t)} \rv}$.  We have
\begin{align*}
    -\nabla L_t \cdot \frac{V^{(t)}}{\lv V^{(t)} \rv} & = \frac{1}{\lv V^{(t)} \rv} \sum_{i \in [2p]} \frac{u_i}{n} \sum_{s=1}^n g_{ts} y_s \phi'(v_i^{(t)} \cdot x_s) v_i^{(t)} \cdot x_s \\
    & = \frac{1}{n \lv V^{(t)} \rv} \sum_{s=1}^n g_{ts} y_s \left[ \sum_{i \in [2p]}u_i \phi'(v_i^{(t)}\cdot x_s) (v_i^{(t)} \cdot x_s)\right].
\end{align*}
Note that, by definition of the Huberized ReLU for any $z \in \mathbb{R}$, $\phi(z) \le \phi'(z)z \le \phi(z)+h/2$, and therefore,
\begin{align} 
\nonumber
    -\nabla L_t \cdot \frac{V^{(t)}}{\lv V^{(t)} \rv} 
    & = \frac{1}{n \lv V^{(t)} \rv} \sum_{s=1}^n g_{ts} \left[y_s \sum_{i \in [2p]}u_i \phi(v_i^{(t)}\cdot x_s) \right] \\& \qquad  + \frac{1}{n \lv V^{(t)} \rv} \sum_{s=1}^n g_{ts} \left[ \sum_{i \in [2p]}y_su_i \left(\phi'(v_i^{(t)}\cdot x_s)(v_i^{(t)}\cdot x_s)-\phi(v_i^{(t)}\cdot x_s) \right)\right] \nonumber\\
    \nonumber & \ge \frac{1}{n \lv V^{(t)} \rv} \sum_{s=1}^n g_{ts} \left[y_s \sum_{i \in [2p]}u_i \phi(v_i^{(t)}\cdot x_s) \right]   - \frac{1}{n \lv V^{(t)} \rv} \sum_{s=1}^n g_{ts} \left( \frac{h}{2}\sum_{i \in [2p]}\lvert y_su_i\rvert  \right) \\
    & \overset{(i)}{=} \frac{1}{n \lv V^{(t)} \rv} \sum_{s=1}^n g_{ts} \left[y_s \sum_{i \in [2p]}u_i \phi(v_i^{(t)}\cdot x_s) \right]  - \frac{hp}{n\lv V^{(t)}\rv}\sum_{s=1}^n g_{ts} \nonumber\\
     & \overset{(ii)}{\ge} \frac{1}{n \lv V^{(t)} \rv} \sum_{s=1}^n g_{ts} \left[y_s \sum_{i \in [2p]}u_i \phi(v_i^{(t)}\cdot x_s) \right] - \frac{L_t}{\lv V^{(t)}\rv} \nonumber \\
         & = \frac{1}{n \lv V^{(t)} \rv} \sum_{s=1}^n 
         \frac{y_s f_{V^{(t)}}(x_s)}{1+\exp\left(y_s f_{V^{(t)}}(x_s)\right)}- \frac{L_t}{\lv V^{(t)}\rv}\label{e:lowerboundonthenormofthegradient},
\end{align}
where $(i)$ follows as $|y_s u_i| = 1$ for all $i \in [2p]$ and the inequality in $(ii)$ follows since $g_{ts} \le L_{ts}$ for all samples by Lemma~\ref{l:relationsbetweengradientandloss} and because $h=1/p$.

For every sample $s$, 
$
    L_{ts} = \log\left(1+\exp\left(-y_s f_{V^{(t)}}(x_s) \right)\right)
$
which implies
\begin{align*}
y_s f_{V^{(t)}}(x_s) = \log\left(\frac{1}{\exp(L_{ts})-1} \right) \quad  \mbox{and} \quad \frac{1}{1+\exp\left(y_s f_{V^{(t)}}(x_s)\right)} = 1-\exp(-L_{ts}).
\end{align*}
Plugging this into inequality~\eqref{e:lowerboundonthenormofthegradient} we derive,
\begin{align*}
-\nabla L_t \cdot \frac{V^{(t)}}{\lv V^{(t)} \rv} 
 & \ge \frac{1}{n \lv V^{(t)} \rv} \sum_{s=1}^n \left( 1 - \exp(-L_{ts}) \right)
      \log\left( \frac{1}{\exp(L_{ts}) - 1} \right)-\frac{L_t}{\lv V^{(t)}\rv}.
\end{align*}
Observe that the function $\left(1-\exp(-z)\right)\log\left(\frac{1}{\exp(z)-1}\right)$ is continuous and concave with $\lim_{z \to 0^{+}}(1-\exp(-z))\log\left(\frac{1}{\exp(z)-1}\right)~=~0$. Also recall that $\sum_{s} L_{ts} = L_tn$. Therefore by Lemma~\ref{l:concave.min},
\begin{align}
-\nabla L_t \cdot \frac{V^{(t)}}{\lv V^{(t)} \rv} 
 & \ge \frac{1}{\lv V^{(t)} \rv}\left[ \frac{ 1 - \exp(-L_{t}n)}{n} 
      \log\left( \frac{1}{\exp(L_{t} n) - 1} \right)-L_t\right]. \label{e:pretaylor_lowerbound}
\end{align}
We know that for any $z \in [0,1]$
\begin{align*}
\exp(z) \le 1+2z \quad \mbox{and} \quad  \exp(-z) \le 1-z+z^2.
\end{align*}
Since $L_t \leq \frac{1}{n^{1 + C_1}}$ and $n \geq 2$ for large enough $C_1$ these bounds on the exponential function combined with inequality~\eqref{e:pretaylor_lowerbound} yields
\begin{align*}
-\nabla L_t \cdot \frac{V^{(t)}}{\lv V^{(t)} \rv} 
 & \geq 
     \frac{1}{\lv V^{(t)} \rv} \left[ (L_t - n L_t^2) \log \left( \frac{1}{2n L_t } \right) - L_t \right] \\
    & =
     \frac{1}{\lv V^{(t)} \rv} \left[ L_t \log\left(\frac{1}{L_t}\right)+nL_t^2\log(2n)-L_t(1+\log(2n))-nL_t^2\log\left(\frac{1}{L_t}\right)\right] \\
     & \ge
     \frac{1}{\lv V^{(t)} \rv} \left[ L_t \log\left(\frac{1}{L_t}\right)-L_t(1+\log(2n))-nL_t^2\log\left(\frac{1}{L_t}\right)\right]\\
     & =
     \frac{L_t\log(1/L_t)}{\lv V^{(t)} \rv} \left[ 1-\frac{1+\log(2)+\log(n)}{\log(1/L_t)}-nL_t\right].
\end{align*}
Recalling again that 
$L_t \leq \frac{1}{n^{1 + C_1}}$ and $n \geq 2$,
\begin{align*}
    -\nabla L_t \cdot \frac{V^{(t)}}{\lv V^{(t)} \rv}  & \ge  \frac{L_t\log(1/L_t)}{\lv V^{(t)} \rv} \left[ 1-\frac{1+\log(2)+\log(n)}{(1+C_1)\log(n)}-\frac{1}{n^{C_1}}\right] \\
    & \ge  \frac{L_t\log(1/L_t)}{\lv V^{(t)} \rv} \left[ 1-\frac{1+2\log(2)}{(1+C_1)\log(2)}-\frac{1}{2^{C_1}}\right]\\
    & \ge  \frac{5L_t\log(1/L_t)}{6\lv V^{(t)} \rv},
\end{align*}
where the final inequality holds for a large enough value of $C_1$.
\end{proof}
We are now ready to prove our theorem. 
\subsection{The Proof}
As mentioned above, 
the proof of Theorem~\ref{t:main.general} is by induction. Given the initial weight matrix $V^{(1)}$ and $p$, the values $\widetilde{Q}_1$ and $\widetilde{Q}_2(Q_1)$ can be chosen as stated below:
\begin{eqnarray}
\label{eq:Q_1_def}\widetilde{Q}_1  &=& \min\left\{\frac{1}{30pL_1\log^2\left(1/L_1\right)},\frac{108 \lv V^{(1)}\rv^2}{125L_1 \log^4\left(1/L_1\right) },\frac{e^2}{120p}\right\} \quad \text{and} \\
\label{eq:Q_2_def} \widetilde{Q}_2(Q_1)  &=& \frac{125 Q_1 L_1 \log^4(1/L_1)}{216 \lv V^{(1)}\rv^2} . 
\end{eqnarray}
The proof goes through for any
positive $Q_1 \le \widetilde{Q}_1$ and any positive $Q_2 \le \widetilde{Q}_2(Q_1)$.
Recall that the sequence of step-sizes is given by $\alpha_t = Q_1 \log^2(1/L_t)$. We will use the following multi-part inductive hypothesis:
\begin{enumerate}[({I}1)]
    \item $L_{t} \le \frac{L_1}{Q_2\cdot (t-1)+1};$
    \item $\alpha_t L_t \le \frac{1}{30p};$ 
    \item $\frac{\log^2(1/L_{t})}{\lv V^{(t)}\rv} \ge 
           \frac{\log^2(1/L_1)}{\lv V^{(1)}\rv}$. 
\end{enumerate}

The base case is trivially true for the first and the third part of the inductive hypothesis. It is true for the second part 
since
$Q_1 \le \frac{1}{30pL_1\log^2(1/L_1)}$. 

Now let us assume that the inductive hypothesis holds for a step $t \ge 1$ and
prove that it holds for the next step $t+1$.  We start with Part~I1.
\begin{lemma}
\label{l:L.inductive_step_constant_step_size}

If the inductive hypothesis holds at step $t$ then,
\[
L_{t+1} \leq 
\frac{L_1}{Q_2t + 1}.
\]
\end{lemma}
\begin{proof}
Since $\alpha_t L_t<1/(30p)$ by applying Lemma~\ref{l:L.one_step_improvement}
\begin{align*}
    L_{t+1} \le L_t - \frac{5\alpha_t}{6} \lv \nabla L_t\rv^2.
\end{align*}
 By the lower bound on the norm of the gradient established in Lemma~\ref{l:lower.bound.gradient} since $L_t \le L_1 \le 1/n^{1+C_1}$ we have
\begin{align*}
    L_{t+1} \le L_t - \frac{125\alpha_t L_t^2 \log^2(1/L_t)}{216\lv V^{(t)} \rv^2} &= L_t - \frac{125Q_1 L_t^2 \log^4(1/L_t)}{216\lv V^{(t)} \rv^2}\\
    &\le L_t\left(1 - \frac{125Q_1 L_t \log^4(1/L_1)}{216\lv V^{(1)} \rv^2}\right)\numberthis \label{e:lossdecreasemasterequation},
\end{align*}
where the final inequality makes use of the third part of the inductive hypothesis.
For any $z\ge 0$, the quadratic function 
$$z-z^2\frac{125Q_1  \log^4(1/L_1)}{216\lv V^{(1)} \rv^2}$$
 is a monotonically increasing function in the interval
$$\left[0,\frac{108\lv V^{(1)} \rv^2}{125Q_1  \log^4(1/L_1)}\right].$$
Thus, because 
$L_t \leq \frac{L_1}{Q_2(t-1)+1}$, if $ \frac{L_1}{Q_2(t-1)+1} \leq \frac{108\lv V^{(1)} \rv^2}{125Q_1  \log^4(1/L_1)}$, the RHS of 
\eqref{e:lossdecreasemasterequation}
is bounded above by its value 
when $L_t = \frac{L_1}{Q_2(t-1)+1}$. But this is easy to check: by our choice of the constant $Q_1$ we have,
\begin{align*}
     &Q_1 \le \widetilde{Q}_1 \le  \frac{108\lv V^{(1)} \rv^2}{125 L_1  \log^4(1/L_1)} \\
    & \Rightarrow L_1 \le \frac{108\lv V^{(1)} \rv^2}{125Q_1  \log^4(1/L_1)} \\
    &\Rightarrow \frac{L_1}{Q_2(t-1)+1} \le \frac{108\lv V^{(1)} \rv^2}{125Q_1  \log^4(1/L_1)}. \\
\end{align*}
Bounding the RHS of inequality~\eqref{e:lossdecreasemasterequation}
by using the worst case that $L_t = \frac{L_1}{Q_2(t-1)+1}$, we get
 \begin{align*}
     L_{t+1} & \le \frac{L_1}{Q_2(t-1)+1}\left(1 - \frac{L_1}{Q_2(t-1)+1}\frac{125Q_1 \log^4(1/L_1)}{216\lv V^{(1)} \rv^2}\right)\\
     & = \frac{L_1}{Q_2t+1}\left(\frac{Q_2t+1}{Q_2(t-1)+1}\right)\left(1 - \frac{Q_2}{Q_2(t-1)+1}\frac{125Q_1 L_1 \log^4(1/L_1)}{216Q_2\lv V^{(1)} \rv^2}\right)\\
     & = \frac{L_1}{Q_2t+1}\left(1+\frac{Q_2}{Q_2(t-1)+1}\right)\left(1 - \frac{Q_2}{Q_2(t-1)+1}\frac{125Q_1 L_1\log^4(1/L_1)}{216Q_2\lv V^{(1)} \rv^2}\right)\\
     & \le \frac{L_1}{Q_2t+1}\left(1-\left(\frac{Q_2}{Q_2(t-1)+1}\right)^2\right)\\
     & \hspace{1in} \left(\mbox{since $Q_2 \le \frac{125 Q_1 L_1 \log^4(1/L_1)}{216 \lv V^{(1)}\rv^2}$} \right) \\
     & \le \frac{L_1}{Q_2 t+1}.
 \end{align*}
  This establishes the desired upper bound on the loss at step $t+1$.
\end{proof}
In the next lemma we ensure that the second part of the inductive hypothesis holds. 

\begin{lemma}
\label{l:verify.valid.step_size} Under the setting of Theorem~\ref{t:main.general} if the induction hypothesis holds at step $t$ then,
\begin{align*}
    \alpha_{t+1} L_{t+1} \le \frac{1}{30p}.
\end{align*}
\end{lemma}
\begin{proof}
We know by the previous lemma that if the induction hypothesis holds at step $t$, then $L_{t+1} \le L_t \le 1$. The function $z \log^2(1/z)$ is no more than $4/e^2$ for $z \in (0,1]$. Since $Q_1 \le e^2/(120p)$ we have 
\begin{align*}
    \alpha_{t+1}L_{t+1} = Q_1 L_{t+1}\log^2(1/L_{t+1}) \le \frac{1}{30p}.
\end{align*}
\end{proof}
Finally, we shall establish that the third part of the inductive hypothesis holds.
\begin{lemma} Under the setting of Theorem~\ref{t:main.general} if the induction hypothesis holds at step $t$ then,
\label{l:vbound}
\begin{align*}
    \frac{\log^2\left(\frac{1}{L_{t+1}}\right)}{\lv V^{(t+1)}\rv} \ge \frac{\log^2\left(\frac{1}{L_{1}}\right)}{\lv V^{(1)}\rv}.
\end{align*}
\end{lemma}
\begin{proof}
We know from Lemma~\ref{l:L.one_step_improvement} that $L_{t+1} \le L_t\left(1 - 5\alpha_t \lv \nabla L_t \rv^2/(6L_t)\right)$, and by the triangle inequality $\lv V^{(t+1)} \rv \le \lv V^{(t)} \rv + \alpha_t \lv \nabla L_t \rv$, hence
\begin{align*}
    \frac{\log^2\left(\frac{1}{L_{t+1}}\right)}{\lv V^{(t+1)}\rv} & \ge  \frac{\log^2\left(\frac{1}{L_t\left(1 - \frac{5\alpha_t}{6L_t}\lv \nabla L_t \rv^2 \right)}\right)}{\lv V^{(t)}\rv + \alpha_t \lv \nabla L_t\rv} \\
    & = \frac{\left(\log\left(\frac{1}{L_t}\right)-\log\left(1 - \frac{5\alpha_t}{6L_t}\lv \nabla L_t \rv^2 \right)\right)^2}{\lv V^{(t)}\rv + \alpha_t \lv \nabla L_t\rv} \\
    & = \frac{\log^2\left(\frac{1}{L_t}\right)-2\log\left(\frac{1}{L_t}\right)\log\left(1 - \frac{5\alpha_t}{6L_t}\lv \nabla L_t \rv^2 \right)+\log^2\left(1 - \frac{5\alpha_t}{6L_t}\lv \nabla L_t \rv^2 \right)}{\lv V^{(t)}\rv + \alpha_t \lv \nabla L_t\rv} \\
    & \overset{(i)}{\ge} \frac{\log^2\left(\frac{1}{L_t}\right)\left(1-\frac{2\log\left(1 - \frac{5\alpha_t}{6L_t}\lv \nabla L_t \rv^2 \right)}{\log\left(\frac{1}{L_t}\right)}\right)}{\lv V^{(t)}\rv \left(1+ \frac{\alpha_t \lv \nabla L_t\rv}{\lv V^{(t)} \rv}\right)} \\
    & \overset{(ii)}{\ge} \frac{\log^2\left(\frac{1}{L_t}\right)}{\lv V^{(t)}\rv}\left\{\frac{1+\frac{5\alpha_t \lv \nabla L_t \rv^2}{3L_t\log\left(\frac{1}{L_t}\right)}}{1+ \frac{\alpha_t \lv \nabla L_t\rv}{\lv V^{(t)} \rv}}\right\} 
    \numberthis \label{eq:lower_bound_ratio}
\end{align*}
where in $(i)$ the lower bound follows as we are dropping a positive lower-order term, and $(ii)$ follows since $\log(1-z) \le -z$ for all 
$z < 1$ and 
\begin{align*}
\frac{5\alpha_t}{6L_t}\lv \nabla L_t \rv^2
& \leq \frac{10p\alpha_t  L_t}{6}
  \hspace{1in}
  \mbox{(by Lemma~\ref{l:gradient_norm.upper})} \\
& < 1
\end{align*}
by the inductive hypothesis.
 
 We want the term in curly brackets in inequality \eqref{eq:lower_bound_ratio} to be at least 1, that is, 
\begin{align*}
    & 1+\frac{5\alpha_t \lv \nabla L_t \rv^2}{3L_t\log\left(\frac{1}{L_t}\right)} \ge 1+ \frac{\alpha_t \lv \nabla L_t\rv}{\lv V^{(t)} \rv} \\
    & \Leftarrow \lv \nabla L_t \rv  \ge  \frac{3L_t \log\left(\frac{1}{L_t}\right)}{5\lv V^{(t)} \rv},
\end{align*}
which follows from Lemma~\ref{l:lower.bound.gradient} which ensures that $\lv \nabla L_t\rv \ge 5L_t\log(1/L_t)/(6\lv V^{(t)} \rv)$ (since $5/6 \ge 3/5$). Thus we can infer that
\begin{align*}
      \frac{\log^2\left(\frac{1}{L_{t+1}}\right)}{\lv V^{(t+1)}\rv} \ge   \frac{\log^2\left(\frac{1}{L_{t}}\right)}{\lv V^{(t)}\rv} \ge \frac{\log^2\left(\frac{1}{L_{1}}\right)}{\lv V^{(1)}\rv}.
\end{align*}
This proves that the ratio is lower bounded at step $t+1$ by its initial value and establishes our claim.
\end{proof}

Combining the results of Lemmas \ref{l:L.inductive_step_constant_step_size}, \ref{l:verify.valid.step_size} and \ref{l:vbound} completes the proof of theorem.

\section{Proof of Theorem~\protect\ref{t:main.clusters}}
\label{s:proofdetails_theorem_cluster}

The proof of Theorem~\ref{t:main.clusters} has two parts. First, we analyze the first step and show that the loss decreases by a factor that is exponentially large in $p^{1/2-\beta}$. After this, we complete the proof by invoking Theorem~\ref{t:main.general}.

\subsection{The Effect of the Reduction on the Clusters}
\label{s:prelim.clusters}

We reduce to the case that the bias terms are fixed at
zero in the context of Theorem~\ref{t:main.clusters}.
In this case, we can assume the following without loss of generality:
\begin{itemize}
    \item $b_1,\ldots,b_{2p}$ are fixed to $0$, and not trained,
    \item for all $s$, $\lv x_s \rv = 1$, 
    \item for all $s, q \in [n]$, $x_s \cdot x_{q} \geq 0$,
    \item for all $k, \ell \in [4]$ for $y_k \neq y_{\ell}$, $\mu_k \cdot \mu_{\ell} \leq (1+\Delta)/2$, and
    \item for all $s,k$,\; $s \in \cK_k$, $\lv x_s - \mu_k \rv \leq r/\sqrt{2}$.
\end{itemize}
The details are in Appendix~\ref{a:reduction.clusters}.

\subsection{Analysis of the Initial Step} \label{ss:clustered_init}
Our analysis of the first step will make reference to the set
of hidden units that ``capture'' an example 
by a sufficient margin,
further dividing them into helpful and harmful
units.
\begin{definition}
\label{d:I}
Define 
\begin{align*}
I_{+s}&:= \left\{i\in [2p]: \left(u_i = y_{s}\right) \wedge \left(v_i^{(0)} \cdot x_s \ge h+4 \alpha_0\right)\right\} \; \text{and } \\
I_{-s}&:= \left\{i\in [2p]: \left(u_i = -y_{s}\right) \wedge \left(v_i^{(0)} \cdot x_s \ge h+4 \alpha_0\right)\right\}.
\end{align*}
\end{definition}
Next, we prove that the random initialization satisfies a 
number of properties with high probability.  (We will later
show that they are sufficient for convergence.)
The proof is in Appendix~\ref{a:concentration}. (Recall that $C_2,\ldots,C_4$ are specified in the
statement of Theorem~\ref{t:main.clusters}. We also remind the reader that sufficiently large $C_3$ means that $p$ is sufficiently large.)

\begin{restatable}{lem}{concentrationlemma}
\label{l:concentration}
There exists a real-valued function $\chi$ such that, for all $(x_1,y_1),\ldots,(x_n,y_n)\in \S^{d-1} \times\{-1,1\}$, for all small enough $C_2$ and all large enough $C_3$,
with probability $1 - \delta$ over the draw of $V^{(0)}$, all of the following hold.
\begin{enumerate}
\item \label{i:meanofiplusandiminus} For all $s \in [n]$,
\begin{align*}
    \sum_{i\in I_{+s}} v_i^{(0)} \cdot x_s &\ge p  \chi(h,\alpha_0,\sigma) - 1, \text{ and}\\
    \sum_{i\in I_{-s}} v_i^{(0)} \cdot x_s & \le  p  \chi(h,\alpha_0,\sigma) +1.
\end{align*}
\item \label{i:sizeofthesetiplusandiminus} 
For all samples $s \in [n]$,
\begin{align*}
     (1/2 - o(1))p &\le \lvert I_{+s} \rvert \le (1/2+o(1))p, \text{ and}\\
    (1/2 - o(1))p &\le \lvert I_{-s} \rvert \le (1/2+o(1))p.
\end{align*}
    \item  \label{i:init.lprime}
       For all samples $s \in [n]$ 
\begin{align*}
    \frac{1}{2}-o(1)\le g_{0s} \le \frac{1}{2}+o(1).
\end{align*}
\item \label{i:numberofelementsinthesameclustercaptured} For all clusters $k \in [4]$,
\begin{align*}
    &\left\lvert \left\{i \in [2p]: \forall \; s \in \cK_k,\; i \in I_{+s}) \right\}\right\rvert \ge \left(\frac{1-\sqrt{r}}{2} -o(1)\right)p, \text{ and}\\
    &\left\lvert \left\{i \in [2p]: \forall \; s \in \cK_k,\; i \in I_{-s}) \right\}\right\rvert \ge \left(\frac{1 - \sqrt{r}}{2} -o(1)\right)p.\\
\end{align*}
     \item \label{i:init.captured.not.captured.part2}
         For all pairs $s, q \in [n]$ such that $y_s \neq y_q$,
         \begin{align*}
        \left| \left\{ i \in [2p] :
         \left(i \in I_{+s} \right)
        \wedge \left(v_i^{(0)} \cdot x_q \geq 0\right)
         \right\} \right| &\leq \left(\frac{1}{3} +
            \frac{\Delta}{4}
            + r + o(1)\right) p, \text{ and}\\
         \left| \left\{ i \in  [2p]:
         \left( i\in I_{-s}\right)
        \wedge \left( v_i^{(0)} \cdot x_q \geq 0\right)
         \right\} \right| &\leq \left(\frac{1}{3} +
            \frac{\Delta}{4}
            + r + o(1)\right) p.
         \end{align*}
         \item \label{i:sizeofsetthatfallsinthemiddle} For all samples $s \in [n]$,
         \begin{align*}
             &\left\lvert \left\{i \in [2p]: \left(- \alpha_0\left(\frac{1}{2}+2(\Delta+r)\right)\le v_i^{(0)} \cdot x_s \le h+ 4 \alpha_0 \right)\wedge \left(u_i \neq y_s\right) \right\} \right\rvert \\ 
             &\hspace{2in}\le  
            \left( \frac{\sqrt{2}}{\sigma\sqrt{\pi}} \left(h+5 \alpha_0(2+\Delta+r)\right)
              \right) p.
         \end{align*}
    \item \label{i:init:normboundV} The norm of the weight matrix after one iteration satisfies $ \frac{3}{5}\sqrt{\frac{d}{p^{\beta}}}\le\lv V^{(1)} \rv \le 3\sqrt{\frac{d}{p^{\beta}}}$.
\end{enumerate}
\end{restatable}

\begin{definition}
\label{d:good_run}
If the random initialization satisfies all of the 
conditions 
of Lemma~\ref{l:concentration},
let us refer to the
entire ensuing training process as {\em a good run.}
\end{definition}

Armed with Lemma~\ref{l:concentration}, it suffices to
show that the loss bounds of Theorem~\ref{t:main.clusters} hold
on a good run.  For the rest of the proof, let us assume
that we are analyzing a good run.  
\begin{lemma}
\label{l:loss.one_step} 
For all small enough $C_2 > 0$, all large enough 
$C_3$ and all small enough $C_4 > 0$, 
the loss after the initial step of gradient descent is bounded above as follows:
\begin{align*}
    L_{1} \le \exp\left(-C_4p^{(1/2 - \beta)}\right).
\end{align*}
\end{lemma}
\begin{proof}Let us examine the loss of each example after one step. Consider an example $s \in [n]$. Without loss of generality let us assume that $y_s = 1$ and that it belongs to cluster $\cK_k$:
\begin{align*}
    L_{1s}  = \ln\left(1+\exp\left(- \sum_{i=1}^{2p}u_i \phi(v_i^{(1)} \cdot x_s)\right)\right).
\end{align*}
Since $s$ is fixed and $y_s=1$, we simplify the notation for $I_{+s}$ and $I_{-s}$ and define their complements, dividing the hidden nodes into four groups:
\begin{enumerate}
    \item $I_{+}$ where $v_i^{(0)} \cdot x_s \ge h+ 4\alpha_0$ and $u_i = 1$;
    \item $I_{-}$ where $v_i^{(0)} \cdot x_s \ge h+ 4\alpha_0$ and $u_i = -1$;
    \item $\widetilde{I}_{+}$ where $ v_i^{(0)} \cdot x_s < h+ 4\alpha_0$ and $u_i = 1$;
    \item $\widetilde{I}_{-}$ where $v_i^{(0)} \cdot x_s < h+ 4\alpha_0$ and $u_i = -1$.
\end{enumerate}

We have
\begin{align*}
    L_{1s} &= \ln\left(1+\exp\left(- \sum_{i \in I_+} \phi(v_i^{(1)} \cdot x_s)  +  \sum_{i \in I_-} \phi(v_i^{(1)} \cdot x_s)   \right. \right.\\ &  \qquad \qquad \qquad  \qquad \qquad \qquad \qquad \left.\left. - \sum_{i \in \widetilde{I}_+} \phi(v_i^{(1)} \cdot x_s) +  \sum_{i \in \widetilde{I}_-} \phi(v_i^{(1)} \cdot x_s) \right)\right). \numberthis \label{e:lossdecompositionintermsofgroups}
\end{align*}

By definition of the gradient descent update we have, for each node  $i$,
\begin{align*}
    v_i^{(1)}\cdot x_s = v_i^{(0)} \cdot x_s + \frac{\alpha_0 u_i}{n} \sum_{q: v_i^{(0)} \cdot x_q \ge 0} y_{q}   g_{0q}\phi'(v_i^{(0)} \cdot x_{q})(x_{q} \cdot x_s).
\end{align*}
Note that the groups $I_+$ and $I_-$ are defined such that even after one step of gradient descent, for any node $i \in I_{+} \cup I_{-}$
\begin{align}
    \phi(v_i^{(1)}\cdot x_s) = v_i^{(1)}\cdot x -h/2. \label{e:membership_in_I_plus_minus}
\end{align}
That is, $v_i^{(1)} \cdot x_s$ continues to lie in the linear region of $\phi$ after the first step. To see this, notice that for all $q$,
\[
 g_{0q}, \phi'(v_i^{(0)} \cdot x_{q}),  x_{q} \cdot x_s 
     \in [0,1],  
\]
 and hence $\lvert v_i^{(1)} \cdot x_s - v_i^{(0)}\cdot x_s \rvert \le \alpha_0$.

Our proof will proceed using four steps. Each step analyzes the contribution of nodes in a particular group.  We give the outline here, deferring
the proof of some parts to lemmas that follow.

\emph{Steps 1 and 2:} 
In Lemma~\ref{l:bound_on_contribution_of_Iplus_and_Iminus}
we will show that, for an absolute constant $c$,
\begin{align}\label{e:contributionofIplus}
    \sum_{i \in I_{+}} \phi(v_{i}^{(1)}\cdot x_s) & \ge p \chi(h,\alpha_0,\sigma) + \frac{\alpha_0  p}{48} \left(1-c(\Delta+\sqrt{r}+\epsilon)-o(1)\right)-2
\end{align}
and
\begin{align} \label{e:contributionofIminus}
    \sum_{i \in I_-} \phi(v_i^{(1)} \cdot x_s)  \le p \chi(h,\alpha_0,\sigma) - \frac{\alpha_0  p}{48} \left(1-c(\Delta+\sqrt{r}+\epsilon)-o(1)\right)+2.
\end{align}

\emph{Step 3:} Since the Huberized ReLU is non-negative a simple 
bound on the contribution of nodes in $\widetilde{I}_{+}$ is $\sum_{i \in \widetilde{I}_{+}} \phi(v_i^{(1)} \cdot x_k)\ge 0$. 

\emph{Step 4:} Finally 
in Lemma~\ref{l:bound_on_the_contribution_of_tilde_Iminus} we will show that the contribution of the nodes in $\widetilde{I}_-$ is bounded above by
\begin{align}
    \sum_{i \in \tI_-} \phi(v_i^{(1)} \cdot x_s) &\le \frac{\sqrt{2}}{\sigma\sqrt{\pi}} \left(h+5\alpha_0(2+\Delta+r)\right)^2p. \label{e:contributionoftildeIminus}
\end{align}

Combining the bounds in inequalities \eqref{e:contributionofIplus}, \eqref{e:contributionofIminus} and \eqref{e:contributionoftildeIminus} with the decomposition of the loss in \eqref{e:lossdecompositionintermsofgroups} we infer,
\begin{align*}
    L_{1s}  &\le \ln\left(1+\exp\left(-\frac{\alpha_0  p}{24}
    \left(1-c(\Delta+\sqrt{r}+\epsilon)-o(1)\right) + \frac{\sqrt{2}}{\sigma\sqrt{\pi}} \left(h+5\alpha_0(2+\Delta+r)\right)^2p\right)\right) \\
    & \le \ln\left(1+\exp\left(-\frac{\alpha_0  p}{24} \left[1-c(\Delta+\sqrt{r}+\epsilon)-o(1)\right] \right)\right),
\end{align*}
since $h = o(\alpha_0)$, and $\alpha_0 = o(\sigma)$.
Now since $\Delta,r,\epsilon < C_2$, where $C_2$ is a small enough constant, $\alpha_0 = 1/p^{1/2 + \beta}$ and because $p$ is bigger than a suitably large constant we have,
\begin{align*}
    L_{1s} \le \ln\left(1+\exp\left(-C_4 \alpha_0 p\right)\right) = \ln\left(1+\exp\left(-C_4 p^{1/2 - \beta}\right)\right) \le \exp\left(-C_4p^{1/2 - \beta}\right).
\end{align*}
 Recall that the sample $s$ was chosen without loss of generality above. Therefore, by averaging over the $n$ samples we have
\begin{align*}
    L_{1} = \frac{1}{n}\sum_{s=1}^n L_{1s}\le  \exp\left(-C_4p^{1/2 - \beta}\right)
\end{align*}
establishing our claim.
\end{proof}
Next, as promised in the proof of Lemma~\ref{l:loss.one_step}, we bound the contribution due to the nodes in $I_{+}$ and $I_{-}$ after one step.
\begin{lemma} \label{l:bound_on_contribution_of_Iplus_and_Iminus} Borrowing all notation from the proof of Lemma~\ref{l:loss.one_step} above, for all
small enough $C_2$
and large enough $C_3$, there is an absolute constant $c$ such that, on a good run 
\begin{align*}
     \sum_{i \in I_{+}} \phi(v_{i}^{(1)}\cdot x_s) & \ge p \chi(h,\alpha_0,\sigma) + \frac{\alpha_0  p}{48} \left(1-c(\Delta+\sqrt{r}+\epsilon)-o(1)\right)-2, \text{ and }\\
      \sum_{i \in I_-} \phi(v_i^{(1)} \cdot x_s)  &\le p \chi(h,\alpha_0,\sigma) - \frac{\alpha_0  p}{48} \left(1-c(\Delta+\sqrt{r}+\epsilon)-o(1)\right)+2.
\end{align*}
\end{lemma}
\begin{proof}We begin by analyzing the contribution of nodes in group $I_+$.
\begin{align*}
    & \sum_{i \in I_+} \phi(v_i^{(1)} \cdot x_s) \\
    &   = \sum_{i \in I_{+}} \phi\left(v_i^{(0)} \cdot x_s + \frac{\alpha_0}{n} \sum_{q: v_i^{(0)} \cdot x_{q} \ge 0} y_{q} g_{0q} \phi'(v_i^{(0)} \cdot x_{q}) (x_{q} \cdot x_s)\right) \\
    & = \sum_{i \in I_{+}} \left(v_i^{(0)} \cdot x_s + \frac{\alpha_0}{n} \sum_{q: v_i^{(0)} \cdot x_{q} \ge 0} y_{q} g_{0q} \phi'(v_i^{(0)} \cdot x_{q}) (x_{q} \cdot x_s)\right) -\frac{h \lvert I_{+} \rvert}{2} \\
    & \hspace{1.5in} \left(\mbox{since the $i \in I_+$ satisfy \eqref{e:membership_in_I_plus_minus}}\right) \\
    & = \sum_{i \in I_{+}} \left(v_i^{(0)} \cdot x_s + \frac{\alpha_0 g_{0s}}{n} + \frac{\alpha_0 }{n} \sum_{q \neq s: v_i^{(0)} \cdot x_{q} \ge 0} y_{q} g_{0q} \phi'(v_i^{(0)} \cdot x_{q}) (x_{q} \cdot x_s)\right) -\frac{h \lvert I_{+} \rvert}{2}.
\end{align*}
Since
we are analyzing a good run, Parts~\ref{i:meanofiplusandiminus} and \ref{i:sizeofthesetiplusandiminus}
of 
Lemma~\ref{l:concentration}
imply that
$\sum_{i \in I_{+}} v_i^{(0)} \cdot x_k \ge p\chi(h,\alpha_0,\sigma) - 1$ and that $h\lvert I_{+} \rvert \leq 1/2 + o(1)$, therefore, for $p$ larger than a constant,
\begin{align*}
     &\sum_{i \in I_+} \phi(v_i^{(1)} \cdot x_s)\\ 
     & \ge p \chi(h, \alpha_0,\sigma)  + \frac{\alpha_0}{n} \sum_{i \in I_{+}}\left( g_{0s}   + \sum_{q \neq s: v_i^{(0)} \cdot x_{q} \ge 0} y_{q} g_{0q} \phi'(v_i^{(0)} \cdot x_q) (x_{q} \cdot x_s) \right) -2\\
       & \ge p \chi(h,\alpha_0,\sigma)  + \frac{\alpha_0}{n} \left(\sum_{i \in I_{+}} g_{0s} + \sum_{i \in I_{+}} \sum_{q \in \cK_k - \{ s \} , v_{i}^{(0)}\cdot x_q \ge h} g_{0q}(x_q \cdot x_s)\right.\\& 
       \left. \hspace{1.5in}- \sum_{i\in I_{+}}\sum_{q: v_i^{(0)} \cdot x_{q} \ge 0,y_q = -1}  g_{0q} \phi'(v_i^{(0)} \cdot x_q) (x_{q} \cdot x_s) \right) - 2 \numberthis\label{e:iplusboundmiddle}
       \end{align*}
       where the previous inequality above follows in part by 
       recalling that $s \in \cK_k$ where $y_k = 1$, and
       noting that, since $x_s\cdot x_q \ge 0$ for all pairs, we can ignore contributions that have $y_q =1$. Evolving this further
       \begin{align*}
       \sum_{i \in I_+} \phi(v_i^{(1)} \cdot x_s)
       & \overset{(i)}{\ge} p \chi(h,\alpha_0,\sigma)  + \frac{\alpha_0}{n} \left(\sum_{i \in I_{+}} g_{0s} + \sum_{i \in I_{+}} \sum_{q: q \in \cK_k - \{ s \}, v_{i}^{(0)}\cdot x_q \ge h} g_{0q}(1-2r)\right.\\
       & \left. \hspace{1.5in}- \sum_{i\in I_{+}}\sum_{q : v_i^{(0)} \cdot x_{q} \ge 0,y_q = -1} g_{0q} \phi'(v_i^{(0)} \cdot x_q) (x_{q} \cdot x_s) \right) - 2 \\
       & \overset{(ii)}{\ge} p \chi(h,\alpha_0,\sigma)  + \frac{\alpha_0}{n} \left(\underbrace{\sum_{i \in I_{+}} g_{0s} + \sum_{i \in I_{+}} \sum_{q: q \in \cK_k - \{ s \}, v_{i}^{(0)}\cdot x_q \ge h} g_{0q}(1-2r)}_{=: \Xi_1}\right.\\
       & \left. \hspace{1.5in}- \underbrace{\sum_{i\in I_{+}}\sum_{q : v_i^{(0)} \cdot x_{q} \ge 0, y_q = -1}  g_{0q}  (x_{q} \cdot x_s)}_{=:\Xi_2}\right) - 2, \label{e:lowerboundonthecontributionofiplus} \numberthis
\end{align*}
where $(i)$ follows 
since, when $s$ and $q$ are in the same cluster,
$x_s \cdot x_q \geq 1 - 2 r$
(which is proved in Lemma~\ref{l:aux.innerproductboundonx_sandx_q} 
below)
and $(ii)$ follows since 
$\phi$ is $1$-Lipschitz.
Next we provide a lower bound on the term~$\Xi_1$
\begin{align*}
    \Xi_1 
    &= \sum_{i \in I_{+}} g_{0s} + \sum_{i \in I_{+}} \sum_{q \in \cK_k - \{ s \}, v_{i}^{(0)}\cdot x_q \ge h} g_{0q}(1-2r)\\
    &= g_{0s} \lvert I_{+} \rvert + \sum_{i \in I_{+}} \sum_{q \in \cK_k - \{ s \}, v_{i}^{(0)}\cdot x_q \ge h} g_{0q}(1-2r)\\
    & \overset{(i)}{\ge} \left(\frac{1}{2}-o(1)\right) \lvert I_{+} \rvert \\
    & \hspace{40pt} +\left(\frac{1}{2}-o(1)\right)(1-2r) \sum_{q: q \in \cK_k - \{ s \}}\left\lvert\left\{ i\in [2p]: i\in I_{+} \text{ and } v_i^{(0)} \cdot x_q \ge h\right\}\right\rvert   \\
    & \overset{(ii)}{\ge} \left(\frac{1}{2}-o(1)\right)\left[\left(\frac{1}{2}-o(1)\right)p +(1-2r)(|\cK_k| -1)\left(\frac{1-\sqrt{r}}{2}-o(1)\right)p \right]\\
    &\ge \left(\frac{1}{2}-o(1)\right) |\cK_k| (1-2r)\left(\frac{1-\sqrt{r}}{2}-o(1)\right)p \\
    & \overset{(iii)}{\ge} \left(\frac{1}{2}-o(1)\right)\left(\frac{1}{4}-\epsilon\right)(1-2r)\left(\frac{1-\sqrt{r}}{2}-o(1)\right) np\\
    & \overset{(iv)}{\ge} \frac{1}{16} \left( 1 - c_1 (\sqrt{r} + \epsilon) -o(1)\right)np, \numberthis  \label{e:lowerboundonxi1}
\end{align*}
for an absolute positive constant $c_1$,
where $(i)$ follows since, by
Part~\ref{i:init.lprime}
of Lemma~\ref{l:concentration},
on a good run,
$g_{0s} \ge 1/2-o(1)$ for all samples, $(ii)$ follows by using Parts~\ref{i:sizeofthesetiplusandiminus} and \ref{i:numberofelementsinthesameclustercaptured} of Lemma~\ref{l:concentration}, $(iii)$ is by the assumption that $|\cK_k| \ge (1/4-\epsilon)n$ and the simplification in $(iv)$ follows since both $r,\epsilon<C_2$ for a small enough constant $C_2$.

Now we upper bound $\Xi_2$ to get,
\begin{align*}
    \Xi_2 &= \sum_{i\in I_{+}}\sum_{q : v_i^{(0)} \cdot x_{q} \ge 0, y_q = -1}  g_{0q}  (x_{q} \cdot x_s) \\
    & \overset{(i)}{\le} \left(\frac{1}{2}+o(1)\right)\sum_{i\in I_{+}}\sum_{q : v_i^{(0)} \cdot x_{q} \ge 0, y_q = -1}    x_{q} \cdot x_s \\
    &\overset{(ii)}{\le} \left(\frac{1}{2}+o(1)\right)\sum_{i\in I_{+}}\sum_{q : v_i^{(0)} \cdot x_{q} \ge 0, y_q = -1}   \left(\frac{1+\Delta}{2} +2r\right)\\
    & = \left(\frac{1}{2}+o(1)\right) \left(\frac{1+\Delta}{2} +2r\right) \sum_{q :  y_q = -1} \left\lvert \left\{i\in[2p]: i \in I_{+} \text{ and } v_i^{(0)}\cdot x_q \ge 0 \right\}\right\rvert \\
    & \overset{(iii)}{\le} \left(\frac{1}{2}+o(1)\right) \left(\frac{1+\Delta+4r}{2} \right) 
    (|\cK_3| + |\cK_4|)
    \left( \frac{1}{3} +
            \frac{\Delta}{4}
            + r +o(1)\right)p \\
    & \overset{(iv)}{\le} \left(\frac{1}{2}+o(1)\right) \left(\frac{1+\Delta+4r}{2} \right)\left(\frac{1}{2}+2\epsilon\right)\left( \frac{1}{3} +
            \frac{\Delta}{4}
            + r+o(1)\right)np \\
    & \overset{(v)}{\le} \frac{\left(1+c_2 (\Delta + r + \epsilon) +o(1)\right)np}{24}, \numberthis \label{e:upperboundonthetermxi2}
\end{align*}
for an absolute positive constant $c_2$,
where $(i)$ follows as, by
Part~\ref{i:init.lprime}
of Lemma~\ref{l:concentration},
on a good run, for all samples $g_{0q} \le 1/2+o(1)$, $(ii)$ follows from the fact that,
for $q$ and $s$ from
opposite classes, $x_q \cdot x_s \leq \frac{1 + \Delta}{2} + 2 r$
(which is proved in Lemma~\ref{l:aux.innerproductboundonx_sandx_q}
below), $(iii)$ is obtained by invoking
Part~\ref{i:init.captured.not.captured.part2}
of
Lemma~\ref{l:concentration}, $(iv)$ is by the assumption that 
all clusters have at most $(1/4+\epsilon)n$ 
examples and the simplification in $(v)$ follows since $\Delta,r,\epsilon <C_2$ where $C_2$ is a small enough constant.

Combining the conclusion of inequality~\eqref{e:lowerboundonthecontributionofiplus} with the bounds in \eqref{e:lowerboundonxi1} and \eqref{e:upperboundonthetermxi2} completes the proof of the first part of the lemma:
\begin{align*}
    \sum_{i \in I_{+}} \phi(v_{i}^{(1)}\cdot x_s) & \ge p \chi(h,\alpha_0,\sigma) + \frac{\alpha_0  p}{48} \left(1-c_3(\Delta+\sqrt{r}+\epsilon)-o(1)\right)-2.
\end{align*}

Now we move on to analyzing the contribution of the group $I_-$.
\begin{align*}
     \sum_{i \in I_-} \phi(v_i^{(1)} \cdot x_s) 
    &   \le \sum_{i \in I_{-}} \left(v_i^{(1)} \cdot x_s \right)\\
    & = \sum_{i \in I_{-}} \left(v_i^{(0)} \cdot x_s - \frac{\alpha_0}{n} \sum_{q: v_i^{(0)} \cdot x_{q} \ge 0} y_{q} g_{0q} \phi'(v_i^{(0)} \cdot x_{q}) (x_{q} \cdot x_s)\right) \\
    & \overset{(i)}{\le} \sum_{i \in I_{-}} v_i^{(0)} \cdot x_s - \frac{\alpha_0}{n} \left(\sum_{i \in I_{-}}g_{0s}+\sum_{i \in I_{-}}\sum_{q \in \cK_k - \{ s \}, v_i^{(0)}\cdot x_q \ge h} g_{0q} (x_q\cdot x_s)\right)\\
    &  \hspace{1.2in}+ \frac{\alpha_0}{n} \sum_{i \in I_{-}}\sum_{q : v_i^{(0)} \cdot x_{q} \ge 0, y_{q} = -1}  g_{0q} \phi'(v_i^{(0)} \cdot x_{q}) (x_{q} \cdot x_q)\\
    & \overset{(ii)}{\le} p \chi(h,\alpha_0,p)- \frac{\alpha_0}{n} \left(\sum_{i \in I_{-}}g_{0s}+\sum_{i \in I_{-}}\sum_{q \in \cK_k - \{ s \}, v_i^{(0)}\cdot x_q \ge h} g_{0q} (x_q\cdot x_s) \right. \\&\hspace{1.2in} \left. +  \sum_{i \in I_{-}}\sum_{q : v_i^{(0)} \cdot x_{q} \ge 0, y_{q} = -1}  g_{0q} \phi'(v_i^{(0)} \cdot x_{q}) (x_{q} \cdot x_q)\right) +1,
\end{align*}
where $(i)$ follows by noting that $x_s \cdot x_q \ge 0$ for all pairs, therefore we can ignore contributions that have $y_q = 1$, and $(ii)$ is by 
Part~\ref{i:meanofiplusandiminus} of
Lemma~\ref{l:concentration}. Now by using an argument that is identical to that in first part of the proof that bounded the contribution of $I_{+}$ above starting from inequality \eqref{e:iplusboundmiddle} we conclude
\begin{align*} 
    \sum_{i \in I_-} \phi(v_i^{(1)} \cdot x_s)  \le p \chi(h,\alpha_0,\sigma) - \frac{\alpha_0  p}{48} \left(1-c_3(\Delta+\sqrt{r}+\epsilon)-o(1)\right)+2.
\end{align*}
This establishes our bound on the contribution of the nodes in $I_-$.
\end{proof}
In the following lemma we bound the contribution of the nodes in $\tI_{-}$ defined in the proof of Lemma~\ref{l:loss.one_step}.
\begin{lemma} Borrowing all notation from the proof of Lemma~\ref{l:loss.one_step} above, on a good run, 
\label{l:bound_on_the_contribution_of_tilde_Iminus}
\begin{align*}
    \sum_{i \in \tI_-} \phi(v_i^{(1)} \cdot x_s) \le \frac{\sqrt{2}}{\sigma\sqrt{\pi}} \left(h+5\alpha_0(2+\Delta+r)\right)^2p.
\end{align*}
\end{lemma}
\begin{proof}
  Recalling that $v_i^{(1)}$ is obtained by taking a gradient step
\begin{align*}
    &\sum_{i \in \tI_-} \phi(v_i^{(1)} \cdot x_s)\\   
    & = \sum_{i \in \tI_{-}} \phi\left(v_i^{(0)} \cdot x_s - \frac{\alpha_0}{n} \sum_{q: v_i^{(0)} \cdot x_{q} \ge 0} y_{q} g_{0q} \phi'(v_i^{(0)} \cdot x_{q}) (x_{q} \cdot x_s)\right) \\
    & \overset{(i)}{\le} \sum_{i \in \tI_{-}} \phi\left(v_i^{(0)} \cdot x_s + \frac{\alpha_0}{n} \sum_{q: v_i^{(0)} \cdot x_{q} \ge 0, y_{q}= -1}  g_{0q}\phi'(v_i^{(0)} \cdot x_{q}) (x_{q} \cdot x_s)\right) \\
    & \overset{(ii)}{\le} \sum_{i \in \tI_{-}} \phi\left(v_i^{(0)} \cdot x_s + \frac{\alpha_0}{n} \sum_{q: v_i^{(0)} \cdot x_{q} \ge 0, y_{q}= -1}   (x_{q} \cdot x_s)\right) \\
    & \overset{(iii)}{\le} \sum_{i \in \tI_{-}} \phi\left(v_i^{(0)} \cdot x_s + \frac{\alpha_0}{n}\left( n\left( \frac{1}{2}+\frac{\Delta}{2}+2r\right)\right)\right) \\
    & \le \sum_{i \in \tI_{-}} \phi\left(v_i^{(0)} \cdot x_s + \alpha_0\left( \frac{1}{2}+2\left(\Delta+r\right)\right)\right) \\
    & \overset{(iv)}{\le} \left\lvert \left\{i \in [p+1,\ldots,2p]: -\alpha_{0}\left( \frac{1}{2}+2(\Delta+r)\right) \le v_{i}^{(0)} \cdot x_{s} \le h+4\alpha_0\right\} \right\rvert \\
    & \hspace{1in} \times  \left( h+4\alpha_0+\alpha_{0}\left(\frac{1}{2}+2(\Delta+r)\right)\right) \\
    & \overset{(v)}{\le} \left( \frac{\sqrt{2}}{\sigma\sqrt{\pi}} \left(h+5 \alpha_0(2+\Delta+r)\right)
              \right) p \times \left( h+4\alpha_0+\alpha_{0}\left(\frac{1}{2}+2(\Delta+r)\right)\right) \\ &\le \frac{\sqrt{2}}{\sigma\sqrt{\pi}} \left(h+5\alpha_0(2+\Delta+r)\right)^2 p,
\end{align*}
where $(i)$ follows by discarding the contribution of the examples with the same label $y_{q} = 1$, $(ii)$ is because $g_{0\ell}$ and $\phi'$ are non-negative and bounded by $1$, $(iii)$ follows by the 
bound 
$x_{q} \cdot x_s \le (1+\Delta+4r)/2 $ established in Lemma~\ref{l:aux.innerproductboundonx_sandx_q} below. Inequality $(iv)$ 
follows from the facts that
$\phi(z) = 0$ for all $z < 0$
and $v_i^{(0)} \cdot x_s \leq h + 4 \alpha_0$
for all $i \in \tI_+$,
and finally
$(v)$ follows from 
Part~\ref{i:sizeofsetthatfallsinthemiddle} of
Lemma~\ref{l:concentration}. This establishes the claim.
\end{proof}

\subsection{Proof of Theorem~\ref{t:main.clusters}}

Having analyzed the first step we are now ready to prove Theorem~\ref{t:main.clusters}.

Part~$(a)$ of the theorem follows by invoking Lemma~\ref{l:loss.one_step} that shows that after the first step $L_1 \le \exp\left(-C_4 p^{(1/2-\beta)}\right)$ with probability at least $1-\delta$.

Part~$(b)$ of the theorem shall follow by invoking Theorem~\ref{t:main.general}. Since $p \ge \log^{C_3}(n d/\delta)$ for a large enough constant $C_3$ we know that $L_1 \le 1/n^{1+C_1}$ as required by Theorem~\ref{t:main.general}. 
Also note that Part~\ref{i:init:normboundV} of Lemma~\ref{l:concentration} ensures that on a good run, 
$\frac{3}{5}\sqrt{\frac{d}{p^{\beta}}}\le \lv V^{(1)}\rv \le 3\sqrt{\frac{d}{p^{\beta}}}$.

Set the value of $Q_1 = \frac{e^{2}}{120p}$. (This sets the step-size $\alpha_t = Q_1 \log^2(1/L_t)= \frac{e^2\log^2(1/L_t)}{120p}$.) To invoke Theorem~\ref{t:main.general} we need to ensure that $Q_1 \le \widetilde{Q}_1$ (see its definition in equation~\eqref{eq:Q_1_def}), but this is easy to check since
\begin{align*}
    \widetilde{Q}_1 & = \min\left\{\frac{1}{30pL_1\log^2\left(1/L_1\right)},\frac{108 \lv V^{(1)}\rv^2}{125L_1 \log^4\left(1/L_1\right) },\frac{e^2}{120p}\right\}\\
    & \ge \min\left\{\frac{\exp(C_4p^{(1/2-\beta)})}{30C_4^2p^{2-2\beta}},\frac{972 d\exp(C_4p^{(1/2-\beta)})}{3125 C_4^4 p^{2-3\beta}},\frac{e^2}{120p}\right\}\\&\hspace{0.5in}\mbox{(since $L_1\le \exp\left(-C_4p^{(1/2-\beta)}\right)$ and $\lv V^{(1)}\rv\ge \frac{3}{5}\sqrt{\frac{d}{p^{\beta}}}$)}\\
    & = \frac{e^2}{120p},
\end{align*}
where the final equality holds since $p \geq \log^{C_3} d$. Next we set $Q_2 = \widetilde{Q}_2(Q_1)$ (recall its definition from equation~\eqref{eq:Q_2_def} above):
\begin{align*} 
    Q_2 = \widetilde{Q}_2(Q_1)=  \frac{125  Q_1 L_1  \log^4(1/L_1)}{216 \lv V^{(1)}\rv^2}.
\end{align*}
With these valid choices of $Q_1$ and $Q_2$ we now invoke Theorem~\ref{t:main.general} to get that, for all $t > 1$ 
\begin{align*}
    L_{t}  & \le \frac{L_1}{Q_2\cdot (t-1)+1} \\
         & \le \frac{L_1}
                  {\frac{125  Q_1 L_1 \log^4(1/L_1) p^{\beta}(t-1)}{1944 d} +1} \hspace{0.5in}\mbox{(since $\lv V^{(1)}\rv\le 3\sqrt{\frac{d}{p^{\beta}}}$)}\\
        & = \frac{L_1}
                  {\frac{c_1 L_1 \log^4(1/L_1) (t-1)}{d p^{1-\beta}} +1} \\
            & \leq \frac{L_1}
                  {\max \left\{\frac{c_1 L_1\log^4(1/L_1)  (t-1)}{dp^{1-\beta}}, 1 \right\}} \\
        & = \min \left\{ \frac{d p^{1-\beta}}{c_1 \log^4(1/L_1) (t-1)}, L_1 \right\}\\
        & \le \min \left\{ \frac{d p^{1-\beta}}{c_1 C_4^4p^{2-4\beta} (t-1)}, L_1 \right\}\\
        & = \min \left\{ \frac{d }{c_2 p^{1-3\beta} (t-1)}, L_1 \right\}.
\end{align*}
Combining this with Part~$(a)$,
together with the assumption
that $p \geq \log^{C_3} d$, proves Part~$(b)$.

\section{Simulations}
\label{s:simulations}
In this section, we experimentally verify the convergence results of Theorem~\ref{t:main.general}. We performed 100 rounds of batch
gradient descent to minimize the
softmax loss on random training data.
The training data was for a two-class
classification problem.  There were
$128$ random examples drawn from a distribution
in which each of two equally likely classes was distributed as a
mixture of Gaussians whose centers
had an XOR structure:  the positive
examples came from an
equal mixture of
\[
\cN\left(\left(\sqrt{\frac{1}{2}},-\sqrt{\frac{1}{2}}\right), \frac{I}{100}\right)
\mbox{ and } \cN\left(\left(-\sqrt{\frac{1}{2}},\sqrt{\frac{1}{2}}\right), \frac{I}{100} \right),
\]
and the negative examples came from an equal mixture of 
\[
\cN\left(\left(\sqrt{\frac{1}{2}},\sqrt{\frac{1}{2}}\right), \frac{I}{100}\right)
\mbox{ and } \cN\left(\left(-\sqrt{\frac{1}{2}},-\sqrt{\frac{1}{2}}\right), \frac{I}{100}\right).
\]
The number $p$ of hidden units per class
was $100$.  The activation functions
were Huberized ReLUs with $h = 1/p$.
The weights were initialized using
$\cN(0,(4p)^{-5/4})$ and the initial step
size was $(4p)^{-3/4}$.  
(These correspond to the choice
$\beta = 1/4$ in Theorem~\ref{t:main.clusters}.)
For the other
updates, the step-size on iteration
$t$ was $\log^2(1/L_t)/p$.  
The process
of randomly generating data, randomly
initializing a network, and running
gradient descent was repeated 5 times,
and the curves of training error as
a function of update number are plotted
in Figure~\ref{f:xor}.  The decrease
in the loss with the number of iterations
is roughly in line with our upper bounds.
\begin{figure}
  \centering
  \includegraphics[width=4in]{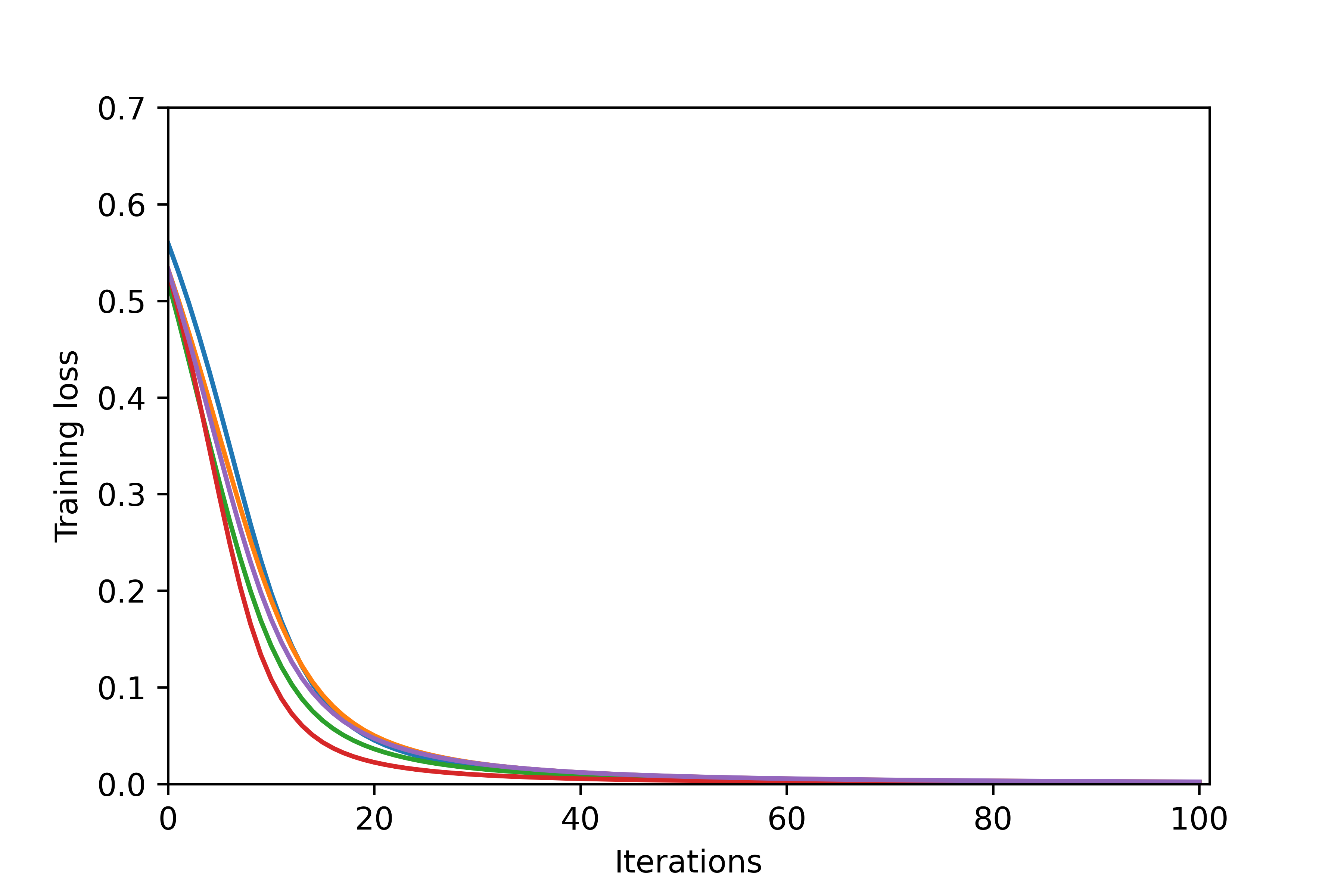}
  \caption{\label{f:xor} Training
  loss as a function of the number 
  of gradient descent steps
  for data distributed as mixtures of Gaussians, where
  the means have an XOR structure.  Details
  are in Section~\ref{s:simulations}.}
  \end{figure}
 
 We performed a similar collection
 of simulations, except with a
 different, more challenging, data
 distribution, which we call the
 ``shoulders'' distribution.  The means of the
 mixture components of the positive
 examples were $(1,0)$ and
 $(0,1)$ while the means of mixture components of the negative examples remain at the same place. The positive centers start to crowd the negative
 center $(1/\sqrt{2}, 1/\sqrt{2})$ making
 it more difficult to pick out examples from
 the negative center.  Plots for
 this data distribution, which also
 scale roughly like our upper bounds,
 are shown in Figure~\ref{f:shoulders}.
  \begin{figure}
  \centering
  \includegraphics[width=4in]{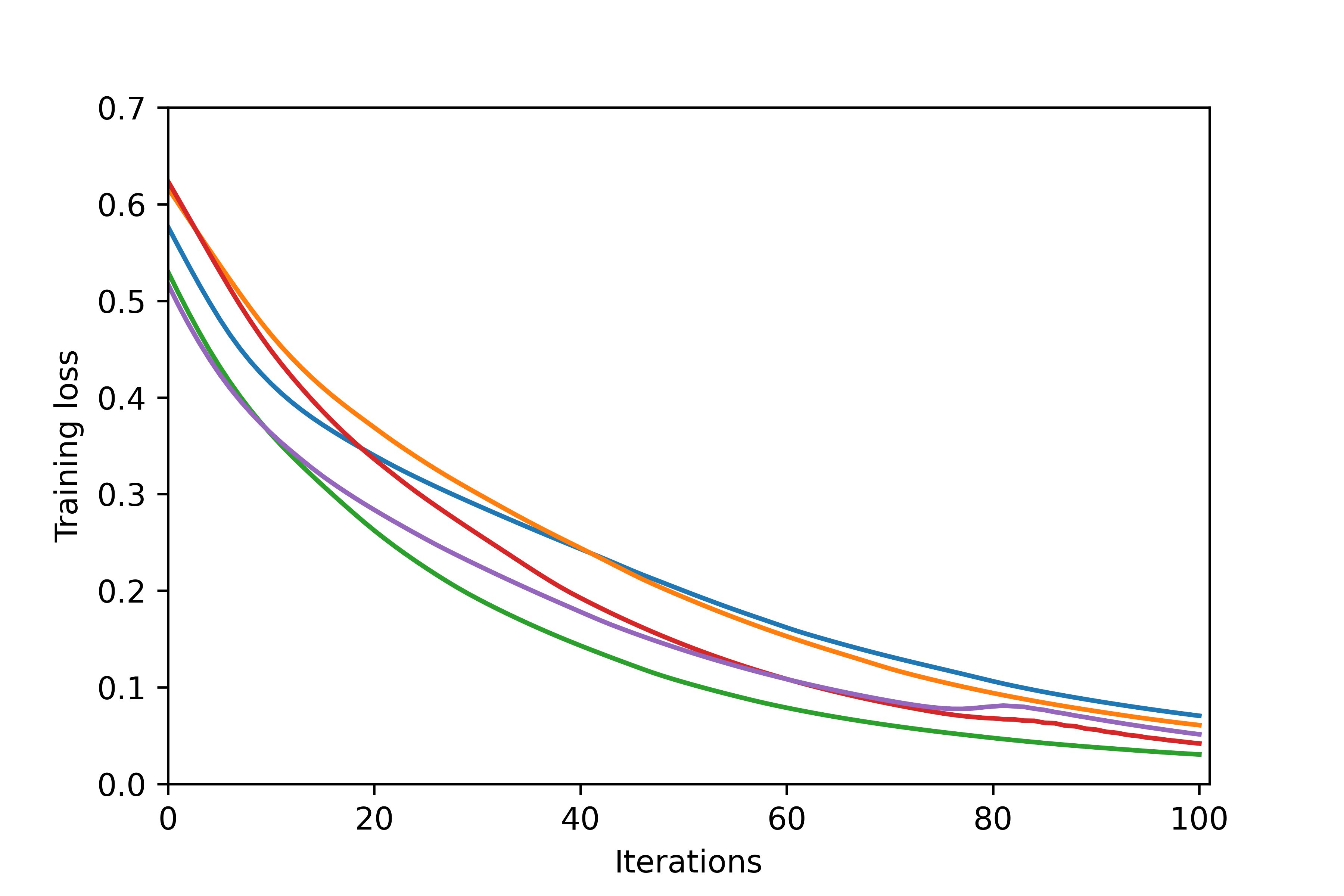}
  \caption{\label{f:shoulders} 
  Training
  loss as a function of the number 
  of gradient descent steps
  for the ``shoulders'' distribution.  Details
  are in Section~\ref{s:simulations}.}
  \end{figure}

Next, we performed ten training runs as described above
for the shoulders data, except that, for five of them,
the Huberized ReLU was replaced by a standard ReLU. The results are in Figure~\ref{f:relu_vs_hrelu}.  While there is evidence
that training with the non-smooth objective arising from
the standard ReLU leads to a limited extent of ``overshooting'', the
shapes of the loss curves agree on a coarser scale.
\begin{figure}
  \centering
  \includegraphics[width=4in]{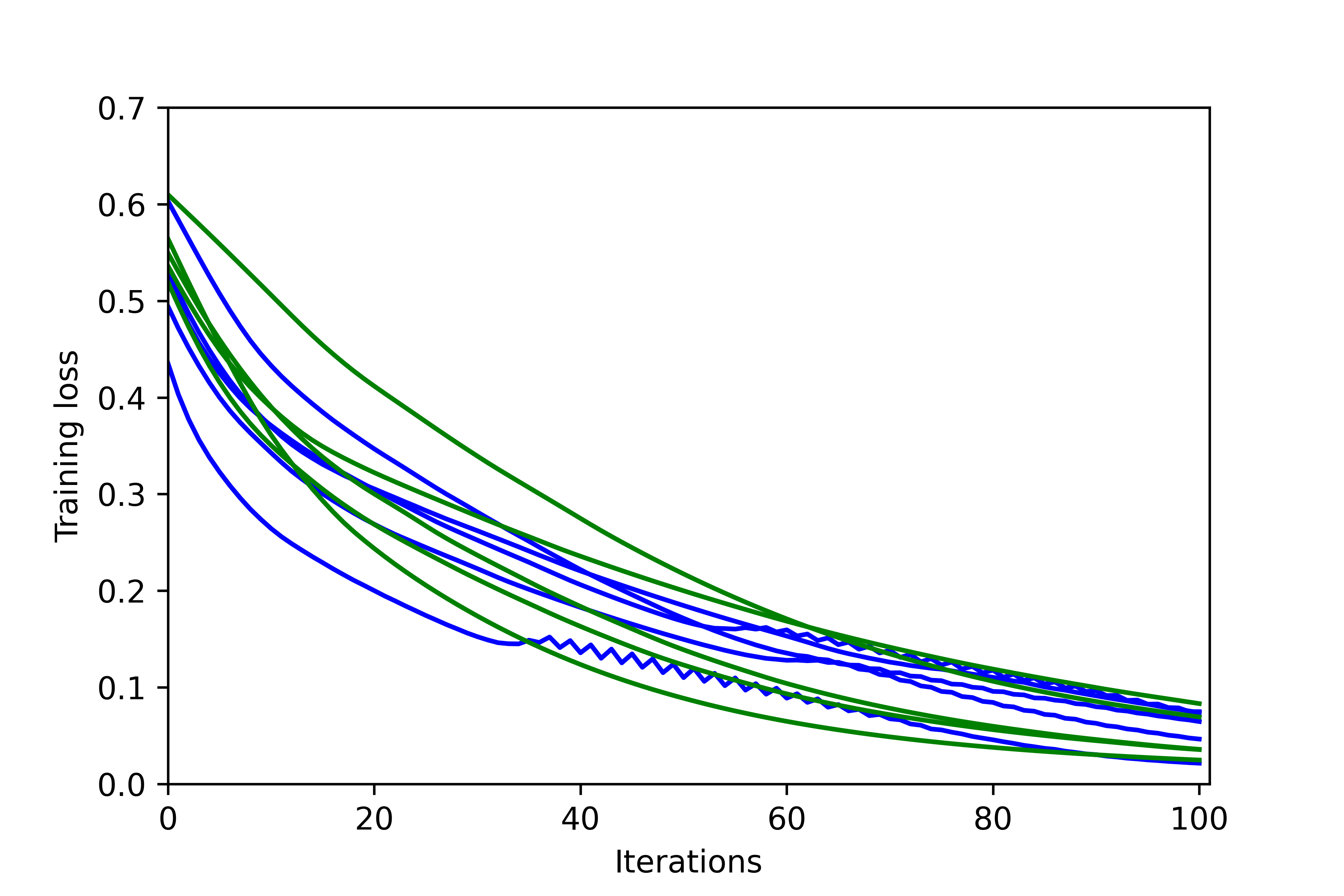}
  \caption{\label{f:relu_vs_hrelu} Training
  loss as a function of the number 
  of gradient descent steps
  with the ReLU (blue) and the Huberized ReLU 
  (green) on the ``shoulders" data.}
  \end{figure}


\section{Additional Related Work}
\label{s:furtherrelatedwork}

\citet{chizat2019lazy} analyzed gradient flow 
for a general class of smoothly parameterized
models, showing that scaling up the initialization,
while scaling down the loss,
ensures that a first-order
Taylor approximation around the initial solution
remains accurate until convergence.

\citet{chizat2020implicit}, building on \citep{chizat2018global,mei2019mean},
show that infinitely wide two-layer squared ReLU networks trained with gradient flow on the logistic loss leads to a max-margin classifier in a particular non-Hilbertian space of functions.
(See also the videos in a talk about this work \citep{chizat20msri}.)
\citet{brutzkus2018sgd} show that finite-width two-layer leaky ReLU networks can be trained up to zero-loss using stochastic gradient descent with the hinge loss, when the underlying data is linearly separable.

The papers \citep{brutzkus2019larger,wei2019regularization,ji2019polylogarithmic} identify when it is possible to efficiently learn XOR-type data using neural networks with stochastic gradient descent on the logistic loss. 

\citet{CCGZ20} analyzed regularized training with gradient flow on infinitely wide networks.  
When training is regularized, the weights also may travel far from their initial values.

Our study is motivated in part by the line of work that has emerged which emphasizes the need to understand the behavior of interpolating (zero training loss/error) classifiers and regressors~\citep[see, e.g.,][among others]{zhang2016understanding,belkin2019reconciling}. A number of recent papers have analyzed the properties of interpolating methods in linear regression \citep{hastie2019surprises,bartlett2020benign,muthukumar2020harmless,tsigler2020benign,bartlett2020failures}, linear classification \citep{montanari2019generalization,chatterji2020finite,liang2020precise,muthukumar2020classification,hsu2020proliferation}, kernel regression \citep{liang2020just,mei2019generalization,liang2020MultipleDescent} and simplicial nearest neighbor methods \citep{belkin2018overfitting}.

Also related are the papers that 
study
the implicit bias of gradient methods \citep{neyshabur2017implicit,soudry2018implicit,ji2018risk,gunasekar2018characterizing,gunasekar2018implicit,li2018algorithmic,arora2019implicit,ji2019gradient}.

A number of recent papers also theoretically study the optimization of neural networks 
including
\citep{andoni2014learning,li2017convergence,ZS0BD17,ZLWJ17a,ge2018learning,ZPS18,du2018gradient,safran2018spurious,zhang2019learning,arora2019fine,DBLP:conf/nips/Daniely20,DBLP:conf/nips/DanielyM20,DBLP:conf/colt/BreslerN20}.

In particular, the proof of \citet{DBLP:conf/nips/DanielyM20} demonstrated
that the first iteration of gradient descent learned useful features for
the parity-learning problem studied there.

\section{Discussion} \label{s:discussion}

We demonstrated that gradient descent drives the logistic loss of finite-width two-layer Huberized ReLU networks to zero if the initial loss is small enough. This result makes no assumptions about the width or the number of samples. We also showed that when the data is structured, and the data satisfies  certain cluster and separation conditions, random initialization followed by gradient descent drives the loss to zero.

After a preliminary version of this paper was posted on arXiv
\citep{chatterji2020does}, related results were obtained~\citep{CLB21colt} for
deep networks with smoothed approximations to the ReLU, under conditions that include Swish. This analysis included adapting the NTK techniques to these activation functions.  
This provides a broader
set of circumstances under which Theorem~\ref{t:main.general} of
this paper can be applied.

Another interesting way forward would be to examine whether the loss can be shown to decrease super-polynomially with the width when there are more than two clusters per label or if the number of samples per cluster is imbalanced.

It would be interesting to see if the corresponding results hold for ReLU activations, which, despite the success of Swish, remain popular.

Now that we have established conditions under which
gradient descent can drive the training error to
zero, future work could study the implicit bias of this limit 
and potentially use this to study the generalization behavior of the final interpolating solution. 
One step towards this could be establish a more precise directional alignment result to show that gradient descent maximizes the margin of Huberized ReLU networks for logistic loss \citep[as][did for ReLU networks trained using gradient flow]{lyu2019gradient,ji2020directional}.

Theorems~\ref{t:main.general} and \ref{t:main.clusters} use
a concrete choice of a learning rate schedule (at least,
up to a constant factor).  We believe that our techniques can
be extended to apply to a wider variety of learning-rate
schedules, with corresponding changes to the convergence rate. 

In our paper, we assumed that the features are all unit-length vectors to simplify the proofs. We believe that the results of Theorem~\ref{t:main.general} can be easily extended to the case where the features have arbitrary bounded lengths. We also expect that the results of Theorem~\ref{t:main.clusters} can be extended to the case where the examples in the four clusters are drawn from sub-Gaussian distributions with suitably small variances.


\subsection*{Acknowledgements}

We thank Zeshun Zong for alerting us to a mistake in
an earlier version of this paper.

We gratefully acknowledge the support of the NSF through grants DMS-2031883 and DMS-2023505 and the Simons Foundation through award 814639.

\appendix

\section{Reduction to the Case of No Bias}
\label{a:reduction.general}

Denoting the components of $x_s$ by $x_{s1},\ldots,x_{sd}$,
define $\tx_s = (x_{s1}/\sqrt{2},\ldots,x_{sd}/\sqrt{2},1/\sqrt{2})$.
We consider the process of training a model using
$(\tx_1,y_1),\ldots,(\tx_n,y_n)$.  

Consider 
\[
(\theta^{(1)},\tV^{(1)}),(\theta^{(2)},\tV^{(2)}),\ldots
\]
defined as follows.  
First, $\theta^{(1)}, \theta^{(2)},\ldots$ are generated 
as described in Section~\ref{s:defs}.  
Each row $\tv_i^{(1)}$ of $\tV^{(1)} \in \R^{2p \times (d+1)}$
is $\sqrt{2} (v_{i1}^{(1)},\ldots,v_{id}^{(1)},b_i^{(1)})$.

Define $\tL$ to be, informally, $L$, but without the bias terms, and
applied to 
\[
(\tx_1,y_1),\ldots,(\tx_n,y_n).
\]
That is
\[
\tL(\widetilde{V}) := \frac{1}{n}\sum_{s=1}^n \ln\left(1+\exp\left(-y_s 
\sum_{i=1}^{2p} u_i \phi(\tv_{i} \cdot \tx_s)
\right)\right).
\]
Then, for $\talpha_1 = 2\alpha_1, \talpha_2 = 2 \alpha_2, \ldots > 0$, we define
$\tV^{(2)}, \tV^{(3)}, \ldots$ to be the iterates of gradient descent applied to
$\tL$, except replacing
$\alpha_1, \alpha_2,\ldots$ by $\talpha_1, \talpha_2,\ldots$.

We claim that, for all $t$, 
\begin{itemize}
\item for all $i$, $\tv_i^{(t)} = \sqrt{2} (v_{i1}^{(t)},\ldots,v_{id}^{(t)},b_i^{(t)})$
\item for all $i$ and all $s$, $\tv_i^{(t)} \cdot \tx_s = v_i^{(t)} \cdot x_s + b_i^{(t)}$,
so that $\tL(\tV^{(t)}) = L(\theta^{(t)})$.
\end{itemize}

The first condition is easily seen to imply the second.  Further, the first condition
holds at $t=1$ by construction.  What remains is to prove that the
inductive hypothesis for iteration $t$ implies the first condition at iteration
$t+1$.  If $f_V$ is the function computed by the network with weights $V$ and no
biases, we have
\begin{align*}
\tv_i^{(t+1)} 
  & = \tv_i^{(t)} 
     + \talpha_t 
        \sum_{s=1}^n
        \frac{ 1 }{1 + \exp(y_s f_{\tV}^{(t)} (\tx_s))}
        \phi'(\tv_i^{(t)} \cdot \tx_s) y_s u_i \tx_s 
\\
          & = \sqrt{2} (v_{i1}^{(t)},\ldots,v_{id}^{(t)},b_i^{(t)})
     + \talpha_t 
        \sum_{s=1}^n
        \frac{ 1 }{1 + \exp(y_s f_{\theta}^{(t)} (x_s))}
        \phi'(v_i^{(t)} \cdot x_s + b_i^{(t)}) y_s u_i \tx_s \\
 & \hspace{2.5in} \mbox{(by the inductive hypothesis)} \\
 & = \sqrt{2} (v_{i1}^{(t+1)},\ldots,v_{id}^{(t+1)},b_i^{(t+1)}),
\end{align*}
because $\tx_s = (x_{s1}/\sqrt{2},\ldots,x_{sd}/\sqrt{2},1/\sqrt{2})$ and $\talpha_t = 2 \alpha_t$, completing the induction.

Finally, note that
\[
\tx_s \cdot \tx_q = \frac{x_s \cdot x_q + 1}{2} \geq 0,
 \]
 since $x_s$ and $x_q$ are unit length.

\section{Omitted Proofs from Section~\ref{ss:technicaltool}} \label{a:omiittedproof_main}
In this section we provide proofs of Lemmas~\ref{l:pl}-\ref{l:L.one_step_improvement}.
\subsection{Proof of Lemma~\ref{l:pl}}
\label{a:pl}
\pllemma*
\begin{proof}
For any $W \in [V^{(t)},V^{t+1}]$ we have that
\begin{align*}
    \nabla_W L(W) - \nabla L_t & = \int_{0}^{1} \left(\nabla^2_{\bar{W}} L|_{\bar{W} = sV^{(t)}+(1-s)W}\right) (W-V^{(t)}) \; \mathrm{d} s,
\end{align*}
where, as stated above, the weak Hessian is defined using the weak derivative $\gamma$ of $\phi'$. 
Thus,
\begin{align*}
    \lv \nabla_W L(W) - \nabla L_t\rv &\le \left[\sup_{s\in [0,1]}\left\lv \nabla^2_{\bar{W}} L|_{\bar{W} = sV^{(t)}+(1-s)W} \right\rv_{op}\right]\lv W-V^{(t)}\rv \le M\lv W-V^{(t)}\rv.
\end{align*}
This shows that along the line segment joining $V^{(t)}$ to $V^{(t+1)}$ the function is $M$-smooth.  Therefore, by using a standard argument~\citep[see, e.g,][Lemma~3.4]{bubeck2015convex} we get that
\begin{align*}
    L_{t+1} &\le L_t + \nabla L_t \cdot (V^{(t+1)}-V^{(t)}) + \frac{M}{2} \lv V^{(t+1)} - V^{(t)} \rv^2\\
    & = L_t - \alpha_t \lv \nabla L_t\rv^2 + \frac{\alpha_t^2 M}{2} \lv \nabla L_t\rv^2 \\
    & = L_t - \alpha_t \left(1-\frac{\alpha_t M}{2}\right)\lv \nabla L_t\rv^2\\
    & \le L_t - \frac{5\alpha_t \lv \nabla L_t\rv^2}{6}.
\end{align*}
This completes the proof.
\end{proof}
\subsection{Proof of Lemma~\ref{l:L_smooth}}
\label{a:L_smooth}
\hessianlemma*
\begin{proof}
We know that the gradient of the loss with respect to $v_i$ is
\begin{align*}
    \nabla_{v_i} L = \frac{-u_i}{n} \sum_{s=1}^n  \frac{\phi'(v_i \cdot x_s) y_s x_s}{1+\exp\left( y_s f_V(x_s)\right)}.
\end{align*}
The weak Hessian $\nabla^2 L$ is a block matrix with $4 p^2$ blocks, where
the $(i,j)^{th}$ block is $\nabla_{v_i} \nabla_{v_j} L$.

First, if $i \neq j$
\begin{align} \label{e:hessian_off_diagonal}
\nabla_{v_i} \nabla_{v_j} L = \frac{u_i u_j}{n} \sum_{s=1}^n  \frac{\phi'(v_i \cdot x_s) \phi'(v_j \cdot x_s) \exp(y_s f_{V}(x_s))  }{\left(1+\exp\left( y_s f_V(x_s)\right) \right)^2}x_s x_s^{\top}.
\end{align}

If $i = j$, 
\begin{align} \label{e:hessian_diagonal}
\nabla^2_{v_i} L = \frac{1}{n} \sum_{s=1}^n  \left[\frac{-u_i\gamma(v_i \cdot x_s) y_s }{1+\exp\left( y_s f_V(x_s)\right)}+ \frac{ \phi'(v_i \cdot x_s)^2\exp(y_s f_{V}(x_s)) }{\left(1+\exp\left( y_s f_V(x_s)\right)\right)^2}  \right]x_s x_s^{\top} .
\end{align}
By definition of the operator norm, 
\begin{equation}
\label{e:op.norm.def}
\lv \nabla^2_V L \rv_{op} = \sup_{a:  \lv a \rv = 1} \left\lv \left(\nabla^2_V L\right) a \right\rv.
\end{equation}
Let $a$ be a unit length member of $\R^{2 p (d+1)}$
and let us decompose $a$ into segments $a_1,\ldots,a_{2p}$ of $(d+1)$ components
each, so that $a$ is the concatenation of $a_1,\ldots,a_{2 p} \in \R^{d+1}$.
Note that $\sum_{i=1}^{2p} \lv a_i \rv^2 = 1$.

The squared norm of $(\nabla^2L)a$ is
\begin{align*}
 \left\lv (\nabla^2 L) a\right\rv^2  &= \sum_{i \in [2p]} \left\lv \sum_{j \in [2p]} \left(\nabla_{v_i}\nabla_{v_j} L\right) a_j\right\rv^2 \\
 &= \sum_{i,j,k \in [2p]} a_k^{\top}\left[ \left(\nabla_{v_i}\nabla_{v_k}L \right)\left(\nabla_{v_i}\nabla_{v_j}L \right)\right]a_j \\
 &\le \sum_{i,j,k \in [2p]} \lv a_j \rv \lv a_k \rv \lv \nabla_{v_i}\nabla_{v_k}L \rv_{op} \lv \nabla_{v_i}\nabla_{v_j}L \rv_{op}. \numberthis \label{e:hessian_upper_bound_midway}
\end{align*}  
By definition of the Huberized ReLU (in equation~\eqref{e:helu}) and its weak Hessian (in equation~\eqref{e:hessian_helu}) we know that, for any $z \in \R$, $\lvert \phi'(z) \rvert < 1$ and $\lvert \gamma(z) \rvert < 1/h$. Further, by Lemma~\ref{l:relationsbetweengradientandloss}, we know that for all $s$
\begin{align*}
    \frac{\exp(y_s f_V(x_s))}{\left(1+\exp\left( y_s f_V(x_s)\right)\right)^2} \le \frac{1}{1+\exp\left( y_s f_V(x_s)\right)}\le L(V; x_s,y_s) = \ln(1+\exp(-y_sf_V(x_s))).
\end{align*}
Also recall that for all $s\in [n]$,  $\lv x_s\rv =1$ and for all $i \in [2p]$, $\lvert u_i\rvert = 1$. 
Applying these to equation~\eqref{e:hessian_off_diagonal}, when $i \neq j$ we get that
\begin{align}
  \lv \left(\nabla_{v_i} \nabla_{v_j} L \right) \rv_{op}
   \le  
       L, \label{e:off_diagonal_operator_bound}
\end{align}
and, using equation~\eqref{e:hessian_diagonal}, when $i = j$ yields the bound 
\begin{align}
  \lv \nabla^2_{v_i} L \rv_{op}  \le    L\left(1+\frac{1}{h}\right) 
   \le 2  L/h. \label{e:diagonal_operator_bound}
\end{align}
Returning to inequality~\eqref{e:hessian_upper_bound_midway},
\begin{align*}
\left\lv (\nabla^2 L) a\right\rv^2  & \le \sum_{i,j,k \in [2p]} \lv a_j \rv \lv a_k \rv \lv \nabla_{v_i}\nabla_{v_k}L \rv_{op} \lv \nabla_{v_i}\nabla_{v_j}L \rv_{op} \\
& = \sum_{i,j,k \in [2p]: i=j=k} \lv a_j \rv^2  \lv \nabla^2_{v_i}L \rv_{op}^2 \\ & \qquad +  \sum_{i,j,k \in [2p]: (i\neq j) \wedge (i\neq k)} \lv a_j \rv \lv a_k \rv \lv \nabla_{v_i}\nabla_{v_k}L \rv_{op} \lv \nabla_{v_i}\nabla_{v_j}L \rv_{op}\\ & \qquad +  \sum_{i,j,k \in [2p]: i=j\neq k} \lv a_j \rv \lv a_k \rv \lv \nabla_{v_i}\nabla_{v_k}L \rv_{op} \lv \nabla^2_{v_i}L \rv_{op}\\
& \qquad +  \sum_{i,j,k \in [2p]: i=k\neq j} \lv a_j \rv \lv a_k \rv \lv \nabla_{v_i}\nabla_{v_j}L \rv_{op} \lv \nabla^2_{v_i}L \rv_{op}.  \numberthis \label{e:splitintodifferentgroups}
\end{align*}
Recall that $h=1/p$, therefore, by inequality~\eqref{e:diagonal_operator_bound}, the first term in the inequality above can be bounded by
\begin{align*}
     \sum_{i,j,k \in [2p]: i=j=k} \lv a_j \rv^2  \lv \nabla^2_{v_i}L \rv_{op}^2 \le   (2L/h)^2 \sum_{j} \lv a_j \rv^2 = 4L^2p^2.
\end{align*}
Using inequality~\eqref{e:off_diagonal_operator_bound}, the second term in the RHS of inequality~\eqref{e:splitintodifferentgroups} is
\begin{align*}
    &\sum_{i,j,k \in [2p]: (i\neq j) \wedge (i\neq k)} \lv a_j \rv \lv a_k \rv \lv \nabla_{v_i}\nabla_{v_k}L \rv_{op} \lv \nabla_{v_i}\nabla_{v_j}L \rv_{op} \\ &\qquad\qquad\qquad  \le  L^2 \sum_{i,j,k \in [2p]: (i\neq j) \wedge (i\neq k)} \lv a_j \rv \lv a_k \rv\\
     & \qquad\qquad\qquad\le  L^2 \sum_{i,j,k \in [2p]: (i\neq j) \wedge (i\neq k)} \frac{\lv a_j \rv^2 +\lv a_k \rv^2}{2}\\
     & \qquad\qquad \qquad=  \frac{(2p-1)L^2}{2} \sum_{i\in [2p]}\left(\sum_{j \in [2p]: i\neq j } \lv a_j \rv^2 + \sum_{k \in [2p]:  i\neq k}\lv a_k \rv^2\right)\\
     & \qquad\qquad\qquad =  \frac{(2p-1)L^2}{2} \sum_{i\in [2p]}2(1-\lv a_i \rv^2)\\
     & \qquad\qquad\qquad = (2p-1)^2L^2\le 4p^2L^2.
\end{align*}
Finally, the last two terms in inequality~\eqref{e:splitintodifferentgroups} can each be bounded by
\begin{align*}
   \sum_{i,j,k \in [2p]: i=j\neq k} \lv a_j \rv \lv a_k \rv \lv \nabla_{v_i}\nabla_{v_k}L \rv_{op} \lv \nabla^2_{v_i}L \rv_{op} & \le (2L/h)\cdot L \sum_{j,k \in [2p]: j\neq k} \lv a_j \rv \lv a_k \rv \\
   & = 2pL^2 \sum_{j,k \in [2p]: j\neq k} \lv a_j \rv \lv a_k \rv \\
   & \le 2pL^2 \sum_{j,k \in [2p]: j\neq k} \frac{\lv a_j \rv^2 + \lv a_k \rv^2}{2} \\
   &= pL^2 \sum_{j \in [2p]} \left(\lv a_j \rv^2 +\sum_{k\in [2p]:k\neq j} \lv a_k \rv^2 \right)\\
   &= pL^2 \sum_{i \in [2p]} \left(\lv a_j \rv^2 +(1-\lv a_j \rv^2) \right)\\
   & = 2p^2L^2 .
\end{align*}
The bounds on these four terms combined with inequality~\eqref{e:splitintodifferentgroups} tells us that 
\begin{align*}
    \left\lv (\nabla^2 L) a\right\rv^2  & \le 12  L^2 p^2.
\end{align*}
Taking square roots along with the definition of the operator norm in equation~\eqref{e:op.norm.def} completes the proof.
\end{proof}
\subsection{Proof of Lemma~\ref{l:gradient_norm.upper}}
\label{a:gradient_norm.upper}
\gradientupper*
\begin{proof}
Recall the definition of $g_s = \left(1+\exp\left(y_s f_{V}(x_s)\right)\right)^{-1}$. By using the expression for the gradient of the loss
\begin{align*}
    \lv \nabla_V L \rv^2 
    &  = \sum_{i=1}^{2p} \lv \nabla_{v_{i}} L(V) \rv^2 \\
    &  = \frac{1}{n^2}\sum_{i=1}^{2p}   \left\lv \sum_{s=1}^n g_s
      \phi'(v_i \cdot x_s) y_s x_s \right\rv^2
        \\
        & = \frac{1}{n^2}\sum_{i=1}^{2p}  \sum_{s=1}^n \sum_{q=1}^n g_s g_{q}
      \phi'(v_i \cdot x_s) \phi'(v_i \cdot x_{q}) y_s y_{q} x_s \cdot x_{q}. 
\end{align*}
 By definition we know that $\lvert \phi'(v_i \cdot x_s)\rvert < 1$ for all $s\in [n]$ and $\lvert y_s y_q x_s\cdot x_q \rvert \le 1$ for any pair $s,q \in [n]$. Therefore,
\begin{align*}
    \lv \nabla_V L \rv^2 & \le \frac{1}{n^2}\sum_{i=1}^{2p}  \sum_{s=1}^n \sum_{q=1}^n g_s g_{q} = \frac{2p}{n^2}  \sum_{s=1}^n \sum_{q=1}^n g_s g_{q}.
\end{align*} 
Since $g_k, g_{\ell} \leq 1$, this implies
$\lv \nabla_V L \rv^2 \leq 2 p$.

To get the stronger bound when $L(V)$ is small,
by Lemma~\ref{l:relationsbetweengradientandloss}, Part~\ref{i:gradientupperboundedbyloss} we know that $g_s g_{q} \le L_s L_{q}$. Thus,
        \begin{align*}
    \lv \nabla_V L \rv^2& \le \frac{2p}{n^2}\left(  \sum_{s=1}^n \sum_{q=1}^n L_s  L_{q} \right)= 2p\left( \frac{1}{n}\sum_{s=1}^n L_s \right)^2 
     = 2  p L(V)^2
    \end{align*}
completing the proof.
\end{proof}
\subsection{Proof of Lemma~\ref{l:L.one_step_improvement}}
\label{a:L.one_step_improvement}
\onesteplemma*
\begin{proof}
In order to apply Lemma~\ref{l:pl}, we would like to bound
$\lv \nabla^2_W L \rv_{op}$ for all convex combinations $W$
of $V^{(t)}$ and $V^{(t+1)}$.  For 
$N = \ceil{\frac{\sqrt{2p} \lv V^{(t+1)}-V^{(t)}\rv}{L_t}}$, 
we will prove the following by induction:
\begin{quote}
For all $s \in \{ 0,\ldots,N\}$, for all $\eta \in [0,s/N]$,
for $W = \eta V^{(t+1)} + (1 - \eta) V^{(t)}$,
$\lv \nabla^2_W L \rv_{op} \leq 10  p L_t$.
\end{quote}
The base case, where $s = 0$ follows directly from Lemma~\ref{l:L_smooth}.  Now,
assume that the inductive hypothesis holds from some $s$, and,
for $\eta \in (s/N,(s+1)/N]$, consider
$W = \eta V^{(t+1)} + (1 - \eta) V^{(t)}$.
Let $\tW = (s/N) V^{(t+1)} + (1 - s/N) V^{(t)}$.
Applying Lemma~\ref{l:pl} along with the inductive hypothesis, 
$L(\tW) \leq L_t$.
Applying Lemma~\ref{l:gradient_norm.upper},
\begin{align*}
L(W) & \leq L(\tW) + (\sqrt{2p} )\lv W- \tW  \rv \\
  &    \leq L_t +  \frac{\sqrt{2p} \lv V^{(t+1)} - V^{(t)} \rv}{N} \\
  & \le 2L_t.
\end{align*}
Applying Lemma~\ref{l:L_smooth}, this implies
$\lv \nabla^2_W L \rv_{op} \leq 10  p L(V^{(t)})$, completing the proof of the inductive step.

So, now we know that, for all convex combinations
$W$ of $V^{(t)}$ and $V^{(t+1)}$, 
$\lv \nabla^2_W L \rv_{op} \leq 10  p L(V^{(t)})$.
Applying
Lemma~\ref{l:pl}, we have
\begin{align*}
L_{t+1} 
& \le L_t - \frac{5\alpha_t\lv \nabla L_t\rv^2}{6},
\end{align*}
which is the desired result.
\end{proof}
\subsection{Proof of Lemma~\ref{l:concave.min}}
\label{a:concave.min}
\concaveminlemma*
\begin{proof}
If $n=1$ the lemma is trivial.  Consider
the case $n > 1$.
Consider an arbitrary feasible point $z_1,\ldots,z_n$ with $z_1,\ldots,z_n >0$ and $\sum_{i=1}^n z_i = M$.
Assume without loss of generality
that $z_1 \geq z_2\geq \ldots \geq z_n$. For an arbitrarily small $\eta >0$, we claim that the point $z_1 + z_2-\eta, \eta, z_3, \ldots,z_n$
is at least as good.
Since $\psi$ is concave
\begin{align*}
\psi(z_1) 
 &\geq \frac{z_2-\eta}{z_1 + z_2-2\eta} \psi(\eta) + 
    \frac{z_1-\eta}{z_1 + z_2-2\eta}\psi(z_1 + z_2-\eta), \quad 
\mbox{ and } \\
\psi(z_2) 
  &\geq \frac{z_1-\eta}{z_1 + z_2-2\eta} \psi(\eta) + 
    \frac{z_2-\eta}{z_1 + z_2-2\eta}\psi(z_1 + z_2-\eta).
\end{align*}
So by adding these two inequalities we infer
\begin{align*}
 \psi(z_1 + z_2-\eta) + \psi(\eta) + \sum_{i=3}^n \psi(z_i) 
& \leq \psi(z_1) + \psi(z_2) + \sum_{i=3}^n \psi(z_i).
\end{align*}
Repeating this for the other $(n-2)$ components of
the solution, we find that
\begin{align*}
    \psi(M-(n-1)\eta) + (n-1)\psi(\eta) 
& \leq \sum_{i=1}^n \psi(z_i).
\end{align*}
Since $\psi$ is a continuous function by taking the limit $\eta \to 0^{+}$ we get that,
\begin{align*}
    \psi(M) + (n-1)\lim_{\eta \to 0^{+}}\psi(\eta) & \le \sum_{i=1}^n \psi(z_i).
\end{align*}
Given that $z_1,\ldots,z_n$ was an arbitrary feasible point, the previous inequality establishes our claim.
\end{proof}

\section{Reduction to the Case of No Bias with Random Initialization}
\label{a:reduction.clusters}

We
once again consider the process of training a model using
$(\tx_1,y_1),\ldots,(\tx_n,y_n)$, where $\tx_s$
is defined as in Appendix~\ref{a:reduction.general}.

Let $\tsigma = \sqrt{2} \sigma$.  A sample from
$\mathcal{N}(0, \tsigma^2)$ can be generated by sampling
from $\mathcal{N}(0, \sigma^2)$, and scaling the result up by a
factor of $\sqrt{2}$.  

For some $\sigma > 0$, and $\alpha_0, \alpha_1, \alpha_2,\ldots > 0$, $h \geq 0$,
consider the joint distribution on 
\[
(\theta^{(0)},\tV^{(0)}), (\theta^{(1)},\tV^{(1)}),\ldots
\]
defined as follows.  
First, $\theta^{(0)}, \theta^{(1)},\ldots$ are generated 
as described in Section~\ref{s:main.clusters}.  
Each row $\tv_i^{(0)}$ of $\tV^{(0)} \in \R^{2p \times (d+1)}$
is $\sqrt{2} (v_{i1}^{(0)},\ldots,v_{id}^{(0)},b_i^{(0)})$
(so that they are mutually independent draws from $\cN(0, 2\sigma^2)$).

Define $\tL$ as in Appendix~\ref{a:reduction.general}: informally, $L$, but without the bias terms.

Then, for $\talpha_0 = 2\alpha_0, \talpha_1 = 2 \alpha_1, \ldots > 0$, we define
$\tV^{(1)}, \tV^{(2)}, \ldots$ to be the iterates of gradient descent applied to
$\tL$, except replacing
$\alpha_0,\alpha_1, \alpha_2,\ldots$ by $\talpha_0, \talpha_1, \talpha_2,\ldots$.

Arguing as in Appendix~\ref{a:reduction.general}, except starting
with round $0$, we can see that, for all $t$, 
\begin{itemize}
\item for all $i$, $\tv_i^{(t)} = \sqrt{2} (v_{i1}^{(t)},\ldots,v_{id}^{(t)},b_i^{(t)})$
\item for all $i$ and all $s$, $\tv_i^{(t)} \cdot \tx_s = v_i^{(t)} \cdot x_s + b_i^{(t)}$,
so that $\tL(\tV^{(t)}) = L(\theta^{(t)})$.
\end{itemize}

For each cluster $k$, define $\tmu_k$ by
$\tmu_k = \left( \frac{\mu_{k1}}{\sqrt{2}},\ldots,\frac{\mu_{kd}}{\sqrt{2}},\frac{1}{\sqrt{2}}
\right).$  Note that $\lv \tx_s - \tmu_k \rv = \frac{\lv x_s - \mu_k \rv}{\sqrt{2}}$ and, for all clusters $k$ and $\ell$
\[
\tmu_k \cdot \tmu_{\ell} = \frac{\mu_k \cdot \mu_{\ell} + 1}{2}.
\]

\section{Proof of Lemma~\protect\ref{l:concentration}}
\label{a:concentration}
We begin by restating the lemma here. 
\concentrationlemma* 
The different parts of the lemma are proved one at a time in the subsections below. The lemma holds by taking a union bound over all the different parts. Throughout the proof of this lemma we fix the samples $(x_1,y_1),\ldots,(x_n,y_n) \in \S^{d-1}\times \{-1,1\}$. Conditioned on their value, for all $i \in [2p]$ and for all $s\in [n]$, the random variables $v_i^{(0)}\cdot x_s \sim \mathcal{N}(0,\sigma^2)$.
\subsection{Proof of Part~\ref{i:meanofiplusandiminus}}
Consider a fixed sample $s$. Without loss of generality, suppose that $y_s =1$. We want to demonstrate a high probability lower bound on
\begin{align*}
    \sum_{i \in I_{+s}} v_i^{(0)} \cdot x_s  = \sum_{i \in [p]} (v_i^{(0)} \cdot x_s) \mathbf{1}_{i \in I_{+s}}.
\end{align*}
Now the expected value of this sum,
\begin{align*}
    \mathbb{E}\left[\sum_{i \in [p]} (v_i^{(0)} \cdot x_s) \mathbf{1}_{i \in I_{+s}}\right] &= p \mathbb{E}\left[\left(v_1^{(0)} \cdot x_s\right) \mathbf{1}_{1 \in I_{+s}}\right].  
\end{align*}
Choose the function $\chi$ in the statement of the result to be
\begin{align*}
    \chi(h,\alpha_0,\sigma) := 
    \mathbb{E}\left[\left(v_1^{(0)} \cdot x_s\right) \mathbf{1}_{1 \in I_{+s}}\right].
\end{align*}

By applying Hoeffding's inequality~\citep[see][Theorem~2.6.2]{vershynin2018high} (since $v_i^{(0)}\cdot x_s \sim \mathcal{N}(0,\sigma^2)$, the truncated random variable $(v_i^{(0)}\cdot x_s )\mathbf{1}_{i \in I_{+s}}$ is also $c_1\sigma$-sub-Gaussian for an appropriate positive constant $c_1$),
\begin{align*}
    \Pr\left[\sum_{i \in I_{+s}} v_i^{(0)} \cdot x_s \le  \mathbb{E}\left[\sum_{i \in I_{+s}} (v_i^{(0)} \cdot x_s) \right] -\eta p\right] &\le \exp\left(-c_2\eta^2 p/\sigma^2\right)\\&= \exp\left(-c_2\eta^2 p^{2 + \beta}\right)
\end{align*}
since $\sigma = 1/p^{1/2 + \beta/2}$. Setting $\eta = 1/p$ we get
\begin{align*}
    \Pr\left[\sum_{i \in I_{+s}} v_i^{(0)} \cdot x_s \le  \mathbb{E}\left[\sum_{i \in I_{+s}} (v_i^{(0)} \cdot x_s) \right] -1\right] \le \exp\left(-c_2p^{\beta}\right).
\end{align*}
Since $p\ge \log^{C_3}(n d/\delta)$, for any $C_3 \ge c_3/\beta$, for a large enough constant $c_3$ we can establish that 
\begin{align*}
    \Pr\left[\sum_{i \in I_{+s}} v_i^{(0)} \cdot x_s \le  \mathbb{E}\left[\sum_{i \in I_{+s}} (v_i^{(0)} \cdot x_s) \right] -1\right] \le \frac{\delta}{20n}.
\end{align*}
Finally, a union bound over all $n$ samples completes the proof for $I_{+s}$:
\begin{align*}
    \Pr\left[\exists s\in [n]: \sum_{i \in I_{+s}} v_i^{(0)} \cdot x_s \le  \mathbb{E}\left[\sum_{i \in I_{+s}} (v_i^{(0)} \cdot x_s) \right] -1\right] \le \frac{\delta}{20}.
\end{align*}
An identical argument holds for the sum: $\sum_{i \in {I_{-s}}} v_{i}^{(0)} \cdot x_s$ which completes the proof of this part of the lemma.
    
\subsection{Proof of Part~\ref{i:sizeofthesetiplusandiminus}}
Consider a sample $s$. Without loss of generality, suppose that $y_s =1$. Recall the definition of the set 
\begin{align*}
    I_{+s} = \{i \in [2p]: v_i^{(0)} \cdot x_s \ge h+4 \alpha_0 \text{ and } u_i = y_s = 1\}.
\end{align*}
Note that the variable $v_i^{(0)} \cdot x_s$ has a Gaussian distribution with zero-mean and variance $\sigma^2$. Also, recall that $u_i = 1$ for all $i \in [p]$. Therefore, for each $i \in [p]$,
\begin{align*}
    \zeta:= \Pr\left[v_i^{(0)} \cdot x_s \ge h+ 4\alpha_0\right]& = \frac{1}{2} - \Pr\left[v_i^{(0)} \cdot x_s \in [0, h+ 4\alpha_0]\right] \\ &\overset{(i)}{=} \frac{1}{2} - O\left(\frac{h+\alpha_0}{\sigma}\right) \\
    & = \frac{1}{2} - O\left(\frac{1/p + 1/p^{1/2 + \beta}}{1/p^{1/2 + \beta/2}}\right) \\
    & = \frac{1}{2} -o(1),
\end{align*}
where $(i)$ follows by an upper bound of $1/(\sigma\sqrt{2\pi})$ on the density of a Gaussian random variable with variance $\sigma^2$. A Hoeffding bound implies that, for any $\eta >0$
\begin{align*}
    \Pr\left[\Big\lvert \left\lvert I_{+s} \right\rvert - \zeta p \Big\rvert \ge \eta p \right] \le 2\exp\left(-c' \eta^2 p\right).
\end{align*}
Thus by a union bound over all samples
\begin{align*}
    \Pr\left[\exists s \in [n]:\Big| \lvert I_{+s} \rvert - \zeta p \Big| \ge \eta p \right] \le 2n\exp\left(-c' \eta^2 p\right).
\end{align*}
Setting $\eta = 1/p^{1/4}$ and recalling that $\zeta = 1/2 -o(1)$ and $p \ge \ln^{C_3}(nd/\delta)$ completes the argument for the sets $I_{+s}$. An identical argument goes through for the second claim that establishes a bound on the size of the sets $I_{-s}$.

\subsection{Proof of Part~\ref{i:init.lprime}}
By definition $g_{0s} = \left(1+\exp\left(y_s \sum_{i=1}^{2p} u_i \phi(v_i^0 \cdot x_s)\right)\right)^{-1}$. Recall that 
$v_i^{(0)}$ is drawn from a zero-mean Gaussian with variance $\sigma^2 I$. Therefore, for each $i$, $v_i^{(0)} \cdot x_s$ is a zero-mean Gaussian with variance $\sigma^2$ (since $\lv x_s \rv = 1$). For ease of notation let us define $\xi_i := v_i^{(0)} \cdot x_s$. The sigmoid function $1/(1+\exp(t))$ is $1$-Lipschitz. Therefore,
\begin{align*}
    &\left\lvert\frac{1}{1+\exp\left(y_s \sum_{i=1}^{2p} u_i \phi(\xi_i)\right)} - \frac{1}{1+\exp\left(y_s \E\left[ \sum_{i=1}^{2p} u_i \phi(\xi_i)\right]\right)} \right\rvert  \\
    &\qquad \qquad \qquad \qquad \qquad \qquad \qquad \qquad \qquad \qquad \qquad \qquad  \le \left \lvert \sum_{i=1}^{2p} u_i \phi(\xi_i)-\E\left[\sum_{i=1}^{2p} u_i \phi(\xi_i)\right]\right\rvert.
\end{align*}
Additionally, by its definition the Huberized ReLU $\phi$ is also $1$-Lipschitz. Therefore for any pair $z_1, z_2 \in \R^{2p}$
\begin{align*}
    \left\lvert \sum_{i=1}^{2p}u_i \left(\phi(z_{1i})-\phi(z_{2i})\right) \right\rvert
    \le \sum_{i=1}^{2p} \left\lvert \phi(z_{1i})-\phi(z_{2i}) \right\rvert 
    \le \sum_{i=1}^{2p} \left\lvert z_{1i}-z_{2i} \right\rvert 
    &= \lv z_1 - z_2 \rv_1 \\
     &\le \sqrt{2p} \lv z_1 - z_2 \rv.
\end{align*}
Hence, the function $y_s\sum_{i=1}^{2p} u_i \phi(\xi_i)$ is $\sqrt{2p}$-Lipschitz with respect to its argument 
\[
(\xi_1,\ldots,\xi_{2p}).
\]
By the Borell-Tsirelson-Ibragimov-Sudakov inequality for the concentration of Lipschitz functions of Gaussian random variables \citep[see][Theorem~2.4]{wainwright2019high},
\begin{align*}
    \Pr\left[\left\lvert \sum_{i=1}^{2p} u_i \phi(\xi_i)-\E\left[\sum_{i=1}^{2p} u_i \phi(\xi_i)\right] \right\rvert \ge \eta\right] \le 2 \exp\left(-\frac{\eta^2}{4p \sigma^2}\right).
\end{align*}
Recall that $\sigma = 1/p^{1/2 + \beta/2}$, thus,
\begin{align*}
    \Pr\left[\left\lvert \sum_{i=1}^{2p} u_i \phi(\xi_i))-\E\left[\sum_{i=1}^{2p} u_i \phi(\xi_i))\right] \right\rvert \ge \eta\right] \le 2 \exp\left(-c_1p^{\beta}\eta^2\right).
\end{align*}
By choosing $\eta = 1/p^{\beta/4}$, 
\begin{align*}
    \Pr\left[\left\lvert \sum_{i=1}^{2p} u_i \phi(\xi_i))-\E\left[\sum_{i=1}^{2p} u_i \phi(\xi_i))\right] \right\rvert \ge \frac{1}{p^{\beta/4}}\right] \le 2 \exp\left(-c_2p^{\beta/2}\right).
\end{align*}

This tells us that with probability at least $1-2\exp(-c_2 p^{\beta/2})$,
\begin{align*}
    \frac{1}{1+\exp\left(y_s \E\left[ \sum_{i=1}^{2p} u_i \phi(\xi_i))\right]\right)} - \frac{1}{p^{\beta/4}}&\le \frac{1}{1+\exp\left(y_s \sum_{i=1}^{2p} u_i \phi(\xi_i))\right)}\\ &\le  \frac{1}{1+\exp\left(y_s \E\left[ \sum_{i=1}^{2p} u_i \phi(\xi_i))\right]\right)} + \frac{1}{p^{\beta/4}}. \numberthis \label{e:upperbound_gradient_initial_in_terms_of_expectation}
\end{align*}

Next, we calculate the value of $\E\left[ \sum_{i=1}^{2p} u_i \phi(\xi_i))\right]$. Note that all the random variables $\{\xi_i\}_{i\in[2p]}$ are identically distributed. Recall that, $u_i = 1$ if $i \in \{1,\ldots,p\}$ and $u_i = -1$ if $i \in \{p+1,\ldots,2p\}$, thus
\begin{align*}
    \mathbb{E}\left[\sum_{i=1}^{2p} u_i \phi(\xi_i)\right] & = \mathbb{E}\left[\sum_{i=1}^p \phi(\xi_i)\right] - \mathbb{E}\left[\sum_{i=p+1}^{2p} \phi(\xi_i)\right]
    = p\mathbb{E}\left[\phi(\xi_1)]-p\E[\phi(\xi_1)\right] = 0.
\end{align*}
Thus by inequality~\eqref{e:upperbound_gradient_initial_in_terms_of_expectation} we know that with probability at least $1-2\exp(-c_2p^{\beta/2})$
\begin{align*}
    \frac{1}{2}-o(1)\le \frac{1}{1+\exp(y_s \sum_{i=1}^{2p} u_i \phi(v_i^{(0)} \cdot x_s))} \le \frac{1}{2}+o(1).
\end{align*}
A union bound over all $n$ samples completes the proof, since $p \ge \ln^{C_3}(nd/\delta)$ for a large enough constant $C_3$.

\subsection{Proof of Part~\ref{i:numberofelementsinthesameclustercaptured}}
We will first prove the first claim of this part of the lemma. Without loss of generality consider the cluster $\cK_1$ (recall that for all examples $s \in \cK_1$, $y_s =1$). For any pair $s,q \in \cK_1$
\begin{align*}
    &\Pr\left[v_i^{(0)}\cdot x_s \ge h+ 4\alpha_0 \text{ and } v_i^{(0)}\cdot x_q \ge h+4\alpha_0\right]\\ & \ge \Pr\left[v_i^{(0)}\cdot x_s \ge 0 \text{ and } v_i^{(0)}\cdot x_q \ge 0 \right]-\Pr[v_i^{(0)} \cdot x_s \in[0,h+4\alpha_0]]-\Pr[v_i^{(0)} \cdot x_q \in[0,h+4\alpha_0]]\\
    &\overset{(i)}{\ge} \Pr\left[v_i^{(0)}\cdot x_s \ge 0 \text{ and } v_i^{(0)}\cdot x_q \ge 0 \right]-O\left( \frac{h+\alpha_0}{\sigma}\right)\\
    &= \Pr\left[v_i^{(0)}\cdot x_s \ge 0 \text{ and } v_i^{(0)}\cdot x_q \ge 0 \right]-O\left( \frac{1/p + 1/p^{1/2 + \beta}}{1/p^{1/2 + \beta/2}}\right)\\
    & \overset{(ii)}{=} \frac{1 - \arccos(x_s \cdot x_q)/\pi}{2} -o(1)\\
    & \overset{(iii)}{\geq} \frac{1 - \arccos(1 - 2 r)/\pi}{2} -o(1) \\
    &\ge \frac{1 - \sqrt{r}}{2} -o(1)
\end{align*}
where $(i)$ follows by an upper bound of $1/(\sqrt{2\pi}\sigma)$ on the density of a Gaussian random variable, $(ii)$ follows by noting that the conditional probability of $v_i^{(0)} \cdot x_{q} \ge 0$ conditioned on the event that $v_{i}^{(0)}\cdot x_s \ge 0$ is
$1 - \frac{\arccos(x_s  \cdot x_q)}{\pi}$, and
$(iii)$ follows since $x_s  \cdot x_q \geq 1 - 2r$ by 
Lemma~\ref{l:aux.innerproductboundonx_sandx_q}.

Define $\zeta := \Pr\left[v_i^{(0)}\cdot x_s \ge h+ 4\alpha_0 \text{ and } v_i^{(0)}\cdot x_q \ge h+4\alpha_0\right]$. A Hoeffding bound implies that, for any $\eta >0$
\begin{align*}
    \Pr\left[\left\lvert \left\{i \in [p]: \left(v_i^{(0)} \cdot x_s \ge h+4\alpha_0 \right) \wedge \left( v_i^{(0)} \cdot x_q \ge h +4\alpha_0 \right) \right\}\right\rvert \le (\zeta -\eta) p\right] \le \exp\left(-c_1\eta^2 p\right).
\end{align*}
Recall the definition of the set $I_{+s} = \left\{i\in [2p]: (u_i = y_s=1) \wedge \left(v_i^{(0)} \cdot x_s \ge h+4 \alpha_0\right) \right\}$. Therefore, a union bound over all sample pairs $s,q \in \mathcal{K}_1$ implies that
\begin{align*}
     \Pr\left[\left\lvert \left\{i \in [p] : \forall \; s \in \cK_1, i \in I_{+s} \right\}\right\rvert \le (\zeta -\eta) p\right] &\le n^2 \exp\left(-c_1\eta^2 p\right).
\end{align*}
Finally, by taking a union bound over all 4 clusters we get that
\begin{align*}
     \Pr\left[\exists k \in [4]: \left\lvert \left\{i \in [2p] : \forall \; s \in \cK_k, i \in I_{+s} \right\}\right\rvert \le (\zeta -\eta) p\right] &\le 4n^2 \exp\left(-c_1\eta^2 p\right).
\end{align*}
Choosing $\eta = 1/p^{1/4}$, recalling that $\zeta \ge (1-\sqrt{r}-o(1))/2$ and $p \ge \ln^{C_3}(nd/\delta)$ for a large enough constant $C_3$ completes the proof of the first claim.

The second claim of this part of the lemma  follows by an identical argument.

\subsection{Proof of Part~\ref{i:init.captured.not.captured.part2}}
Without loss of generality, consider a node $i \in [p]$ with $u_i=1$ and a fixed pair $s,q \in [n]$ such that $y_s =1$ and  $y_{q}=-1$. Since each of $v_i^{(0)} \cdot x_s$ and $v_i^{(0)} \cdot x_{q}$ are distributed as $\mathcal{N}(0,\sigma^2)$, we have
\begin{align*}
    \Pr[v_i^{(0)} \cdot x_s \geq h + 4 \alpha_0
        \mbox{ and } v_i^{(0)} \cdot x_q \geq 0 ] &\le  \Pr[v_i^{(0)} \cdot x_s \geq 0
        \mbox{ and } v_i^{(0)} \cdot x_q \geq 0 ] \\ 
        &\overset{(i)}{=} \frac{1 - \frac{\arccos(x_s \cdot x_q)}{\pi}}{2}\\ 
        &\overset{(ii)}{\le} \frac{1 - \frac{\arccos((1+\Delta)/2 +2r)}{\pi}}{2} \\
        & \le \frac{1}{3} +
            \frac{\Delta}{4}
            + r,
\end{align*}
if the bound $C_2$ on $r$ and $\Delta$
is small enough,
where $(i)$ follows by noting that the conditional probability of $v_i^{(0)} \cdot x_{q} \ge 0$ conditioned on the event that $v_{i}^{(0)}\cdot x_s \ge 0$ is $1 - \frac{\arccos(x_s \cdot x_q)}{\pi}$, while $(ii)$ follows since by Lemma~\ref{l:aux.innerproductboundonx_sandx_q}, $x_s \cdot x_{q} \le (1+\Delta)/2 +2r$ for samples where $y_s \neq y_{q}$. 

Now, define $\zeta := \Pr[v_i^0 \cdot x_s \geq h + 4 \alpha_0 
        \mbox{ and } v_i^{(0)} \cdot x_q \geq 0 ]$; a Hoeffding bound
        implies that, for any $\eta > 0$
        \[
        \Pr[|\{ i \in [p]: v_i^{(0)} \cdot x_s \geq h + 4 \alpha_0
        \mbox{ and } v_i^{(0)} \cdot x_q \geq 0 ]\}| 
          \geq (\zeta + \eta) p] \leq \exp(-c' \eta^2 p).
        \]
        Choosing $\eta = 1/p^{1/4}$ and recalling that $\zeta \le \frac{1}{3} +
            \frac{\Delta}{4}
            + r$ along with a union bound over the pairs of samples completes the proof of the first claim. An identical argument works to establish the second claim.
        
\subsection{Proof of Part~\ref{i:sizeofsetthatfallsinthemiddle}} 
Without loss of generality, consider a node $i \in[p+1,\ldots,2p]$ with $u_i=-1$ and fix a sample $s$ with $y_s = 1$. Since each $v_i^0 \cdot x_s$ is distributed as $\mathcal{N}(0,\sigma^2)$, we have,
\begin{align} \label{e:upperbound_zeta_part_6}
  \zeta :=  \Pr\left[ -\alpha_0(1/2+2(\Delta+r)) \le v_i^{(0)} \cdot x_s \le h+4\alpha_0\right] \le \frac{1}{\sigma\sqrt{2\pi}} \left(h+5\alpha_0(1+\Delta+r)\right) ,
\end{align}
where the bound on the probability above follows by an upper bound of $1/(\sqrt{2 \pi} \sigma)$ on the density of a Gaussian random variable.
A Hoeffding bound implies that, for any $\eta >0$
\begin{align*}
   & \Pr\left[\lvert \{i \in \{p+1,\ldots, 2p\}: -\alpha_0(1/2+2(\Delta+r))\le v_i^{(0)} \cdot x_s\le h+4\alpha_0 \}\rvert \ge  (\zeta+\eta)p\right] \\ &\le \exp\left(-c'\eta^2 p\right).
\end{align*}
By choosing $\eta = \frac{5\alpha_0}{\sigma \sqrt{2\pi}}= \frac{5}{\sqrt{2\pi}p^{\beta/2}}$, recalling the upper bound on $\zeta$ established in \eqref{e:upperbound_zeta_part_6} and a union bound over all the $n$ samples completes the proof since $p \ge \ln^{C_3} (nd/\delta)$ for a large enough constant $C_3$.

 \subsection{Proof of Part~\ref{i:init:normboundV}} 
 We know that each $v_i^{(0)} \sim \mathcal{N}(0,\sigma^2 I_{(d+1) \times (d+1)})$. Thus by a concentration inequality for the lower tail of a $\chi^2$-random variable with $2(d+1)p$ degrees of freedom~\citep[see][Lemma~1]{laurent2000adaptive} we have that, for any $\eta >0$
 \begin{align*}
       \Pr\left[\frac{\lv V^{(0)} \rv}{\sqrt{2(d+1)p}} \le \sigma\sqrt{1- \eta}\right]\le \exp(-(d+1)p\eta^2/2).
 \end{align*}
Recall that $\sigma = 1/p^{1/2+\beta/2}$, thus by setting
$\eta = 1/32$ we get that
 \begin{align*}
     \Pr\left[\lv V^{(0)} \rv \le \frac{6\sqrt{d+1}}{5p^{\beta/2}}  \right]\le \exp(-c_1 (d+1)p).
 \end{align*}
Since $p \ge \log^{C_3}(nd/\delta)$ for a large enough value of $C_3$, this ensures that 
\[
\lv V^{(0)}\rv \ge 6\sqrt{d+1}/(5p^{\beta/2})
\]
with probability at least $1 - \delta/c_2$. By the reverse triangle inequality,
 \begin{align*}
     \lv V^{(1)} \rv & \ge \lv V^{(0)}\rv - \alpha_0 \lv \nabla L_0 \rv  \overset{(i)}{\ge} \lv V^{(0)}\rv - \alpha_0 \sqrt{2p}   \ge\frac{6\sqrt{d+1}}{5p^{\beta/2}} - \frac{\sqrt{2}}{p^{\beta}}  \overset{(ii)}{\ge} \frac{3}{5}\sqrt{\frac{d}{p^{\beta}}},
 \end{align*}
 where $(i)$ follows by the bound on the norm of gradient established in Lemma~\ref{l:gradient_norm.upper} and $(ii)$ follows since $d\ge 2$ under our clustering assumptions. Hence 
 \begin{align}\label{e:lower_bound_norm}
     \Pr\left[\lv V^{(1)} \rv \ge \frac{3}{5}\sqrt{\frac{d}{p^{\beta}}}\right] \ge 1-\delta/c_2,
 \end{align}
 which establishes the desired lower bound on the norm of $V^{(1)}$. To establish the upper bound we will use the Borell-TIS inequality for Lipschitz functions of Gaussian random variables \citep[see][Example~2.28]{wainwright2019high}. By this inequality we have that, for any $\eta >0$
 \begin{align*}
     \Pr\left[\frac{\lv V^{(0)} \rv}{\sqrt{2(d+1) p}} \ge \sigma(1 + \eta) \right]\le \exp(-(d+1)p\eta^2).
 \end{align*}
Once again because $\sigma = 1/p^{1/2+\beta/2}$, by setting
$\eta = 1/32$ we get that
 \begin{align*}
     \Pr\left[\lv V^{(0)} \rv \ge \frac{3\sqrt{d+1}}{2p^{\beta/2}}  \right]\le \exp(-c_1 (d+1)p).
 \end{align*}
Since $p \ge \log^{C_3}(nd/\delta)$ for a large enough value of $C_3$, this ensures that 
\[
\lv V^{(0)}\rv \le 3\sqrt{d+1}/(2p^{\beta/2}) \le 5\sqrt{d}/2p^{\beta/2}
\]
with probability at least $1 - \delta/c_2$. By the triangle inequality,
 \begin{align*}
     \lv V^{(1)} \rv & \le \lv V^{(0)}\rv + \alpha_0 \lv \nabla L_0 \rv  \overset{(i)}{\le} \lv V^{(0)}\rv + \alpha_0 \sqrt{2p}   \le\frac{5\sqrt{d}}{2 p^{\beta/2}} + \frac{\sqrt{2}}{p^{\beta}}  \le 3\sqrt{\frac{d}{p^{\beta}}},
 \end{align*}
 where $(i)$ follows by the bound on the norm of gradient established in Lemma~\ref{l:gradient_norm.upper}. Hence 
 \begin{align*}
     \Pr\left[\lv V^{(1)} \rv \le 3\sqrt{\frac{d}{p^{\beta}}}\right] \ge 1-\delta/c_2.
 \end{align*}
 Combining this with inequality~\eqref{e:lower_bound_norm} above we get that
\begin{align*}
     \Pr\left[\frac{3}{5}\sqrt{\frac{d}{p^{\beta}}}\le \lv V^{(1)} \rv \le 3\sqrt{\frac{d}{p^{\beta}}}\right] \ge 1-2\delta/c_2
 \end{align*}
 completing our proof.

\section{Auxiliary Lemmas}
In this section we list a couple of lemmas that are useful in various proofs above.

\begin{lemma}\label{l:relationsbetweengradientandloss}
For any $x \in \mathbb{R}^d$ and $y \in \{-1,1\}$ and any weight matrix $V$ we have the following
\begin{enumerate}
    \item \label{i:gradientupperboundedbyloss}$$\frac{1}{1+\exp\left(yf_V(x)\right)} \le  \ln(1+\exp(-yf_V(x))) = L(V;x,y).$$
    \item \label{i:hessiangupperboundedbyloss}$$\frac{\exp\left(yf_V(x)\right)}{\left(1+\exp\left(yf_V(x)\right)\right)^2} \le \frac{1}{1+\exp\left(yf_V(x)\right)} \le L(V;x,y).$$
\end{enumerate}
\end{lemma}
\begin{proof}
Part~$1$ follows since for any $z \in \mathbb{R}$, we have the inequality $(1+\exp(z))^{-1}\le \ln(1+\exp(-z))$. 

Part~$2$ follows since for any $z \in \mathbb{R}^d$, we have the inequality $\exp(z)/\left(1+\exp(z)\right)^2 \le \left( 1+\exp(z)\right)^{-1}$.
\end{proof}

\begin{lemma}\label{l:aux.innerproductboundonx_sandx_q} Given an $r<1$ suppose that for any $k \in [4]$ all samples $s\in \cK_k$ satisfy the bound $\lv x_s - \mu_k \rv \le r/\sqrt{2}$ and for all $k \in [4]$, $\lv \mu_k \rv = 1$. 
\begin{enumerate}
    \item Then for any pair of clusters $\cK_k,\cK_\ell$ such that $y_k \neq y_{\ell}$, and $\mu_k \cdot \mu_{\ell} \le (1+\Delta)/2$ we have, for all $s \in \cK_k$ and $q \in \cK_\ell$
\begin{align*}
    x_s \cdot x_q \le \frac{1+\Delta }{2}+2r.
\end{align*}
\item Given a cluster $\cK_k$, if $s,s' \in \cK_k$ then,
\begin{align*}
    x_s \cdot x_{s'} \ge 1-2r.
\end{align*}
\end{enumerate}
\end{lemma}
\begin{proof}\emph{Proof of Part~1:} By evaluating the inner product and applying the Cauchy-Schwarz inequality
\begin{align*}
    x_s \cdot x_q & = (x_s-\mu_k + \mu_k) \cdot (x_q -\mu_{\ell}+\mu_{\ell}) \\
    & = \mu_k \cdot \mu_{\ell} + (x_s-\mu_k) \cdot \mu_{\ell} + \mu_k \cdot (x_q -\mu_{\ell}) + (x_s-\mu_k) \cdot (x_q -\mu_{\ell}) \\
    & \le \frac{1+\Delta}{2} + \sqrt{2}r + \frac{r^2}{2} \\
    & \le \frac{1+\Delta}{2} + 2r .
\end{align*}
\emph{Proof of Part~2:} Recall that $\lv \mu_k \rv = 1$. Thus, given two samples $s,s' \in \cK_k$,
\begin{align*}
    x_s \cdot x_s' & = (x_s - \mu_k + \mu_k)\cdot (x_{s'} - \mu_k + \mu_k) \\
    & = \mu_k \cdot \mu_k +(x_s-\mu_k) \cdot \mu_k + \mu_k \cdot (x_{s'}-\mu_k) + (x_s - \mu_k) \cdot (x_{s'}- \mu_k) \\
    & \ge 1 - \sqrt{2}r - \frac{r^2}{2} \\
    & \ge 1-2r
\end{align*}
as claimed.
\end{proof}

\printbibliography

\end{document}